%% file: paper.tex
\documentclass{article} 
\usepackage{iclr2020_conference,times}

\usepackage{epsfig,endnotes}
\usepackage{multirow}
\usepackage{subfig}
\usepackage{graphicx}
\usepackage{grffile}

\newcounter{subcopyrightbox@save}
\usepackage[font=bf]{caption}
\usepackage{color, url}
\usepackage{xspace} 

\usepackage{amsmath}

\usepackage{mathrsfs}
\usepackage{amssymb}
\usepackage{amsmath}
\usepackage{amsthm}
\usepackage{epstopdf}
\usepackage{balance}

\usepackage{thm-restate,thmtools}

\newtheorem{proposition}{Proposition}
\newtheorem{definition}{Definition}
\usepackage[ruled,linesnumbered,noend]{algorithm2e}
\usepackage{algorithmic}

\newcommand{\argmin}{\operatornamewithlimits{argmin}}
\allowdisplaybreaks
\newcommand{\myparatight}[1]{\smallskip\noindent{\bf {#1}:}~}
\newenvironment{packeditemize}{\begin{list}{$\bullet$}{\setlength{\itemsep}{0.2pt}\addtolength{\labelwidth}{-4pt}\setlength{\leftmargin}{\labelwidth}\setlength{\listparindent}{\parindent}\setlength{\parsep}{1pt}\setlength{\topsep}{0pt}}}{\end{list}}

\newcommand{\lnorm}[1]
    {\ensuremath{\left\Vert#1\right\Vert}}
\allowdisplaybreaks

\graphicspath{ {../fig/} }

\iclrfinalcopy
\title{Certified Robustness for Top-\emph{k} Predictions against Adversarial Perturbations via Randomized Smoothing}

\author{Jinyuan Jia, Xiaoyu Cao, Binghui Wang, Neil Zhenqiang Gong \\
Duke University\\
\texttt{\{jinyuan.jia,xiaoyu.cao,binghui.wang,neil.gong\}@duke.edu} 
}

\begin{document}

\maketitle

\input{abstract}

\input{introduction}

\input{topk}

\input{exp}

\input{related}

\input{conclusion.tex}

\noindent\textbf{ACKNOWLEDGMENTS}\\
We thank the anonymous reviewers for insightful reviews. This work was supported by NSF grant No. 1937786.

\bibliography{refs,iclr2019_conference}
\bibliographystyle{iclr2020_conference}

\input{appendix-new}
\end{document}

%% file: abstract.tex
\begin{abstract}

It is well-known that  classifiers are vulnerable to {adversarial perturbations}.  
To defend against adversarial perturbations, various \emph{certified} robustness results have been derived. However, existing certified robustnesses are limited to top-1 predictions. In many real-world applications, top-$k$ predictions are more relevant. In this work, we aim to derive certified robustness for top-$k$ predictions. In particular, our certified robustness is based on \emph{randomized smoothing}, which turns any classifier to a new classifier via adding noise to an input example. We adopt randomized smoothing because it is scalable to large-scale neural networks and applicable to any classifier. We derive a \emph{tight} robustness in $\ell_2$ norm for top-$k$ predictions  when using randomized smoothing with Gaussian noise. 
We find that generalizing the certified robustness  from top-1 to top-$k$ predictions faces significant technical challenges. 
We also empirically evaluate our method on CIFAR10 and ImageNet. For example, our method can obtain an ImageNet classifier with a certified top-5 accuracy of 62.8\% when the $\ell_2$-norms of the adversarial perturbations are less than 0.5 (=127/255). Our code is publicly available at: \url{https://github.com/jjy1994/Certify_Topk}. 
\end{abstract}

%% file: introduction.tex
\section{Introduction}
Classifiers are vulnerable to adversarial perturbations \citep{Szegedy14,goodfellow2014explaining,carlini2017towards,jia2018attriguard}. Specifically, given an example $\mathbf{x}$ and a classifier $f$, an attacker can carefully craft a perturbation $\delta$ such that $f$ makes predictions for $\mathbf{x}+\delta$ as the attacker desires. 
 Various empirical defenses (e.g.,~\citet{goodfellow2014explaining,svoboda2018peernets,buckman2018thermometer,ma2018characterizing,guo2018countering,dhillon2018stochastic,xie2018mitigating,song2018pixeldefend}) have been proposed to defend against adversarial perturbations. However, these empirical defenses were often soon broken by adaptive adversaries~\citep{carlini2017adversarial,athalye2018obfuscated}.   
As a response, \emph{certified} robustness (e.g.,~\citet{wong2018provable,raghunathan2018certified,liu2018towards,lecuyer2018certified,cohen2019certified}) against adversarial perturbations has been developed. In particular, a robust classifier verifiably predicts the same top-1 label for data points in a certain region around any example $\mathbf{x}$. 

In many applications such as  recommender systems,  web search, and image classification cloud service~\citep{clarifai_demo,Google_Cloud_Vision}, top-$k$ predictions are more relevant. In particular, given an example, a set of $k$ most likely labels are predicted for the example.  However, existing certified robustness results are limited to top-1 predictions, leaving top-$k$ robustness unexplored. To bridge this gap, we study certified robustness for top-$k$ predictions in this work. Our certified top-$k$ robustness leverages \emph{randomized smoothing}~\citep{cao2017mitigating,cohen2019certified}, which turns any base classifier $f$ to be a robust classifier via adding random noise to an example. For instance,  \citet{cao2017mitigating} is the first to propose randomized smoothing with uniform noise as an empirical defense. We consider random Gaussian noise because of its certified robustness guarantee~\citep{cohen2019certified}.  Specifically, we denote by $p_i$ the probability that the base classifier $f$ predicts label $i$ for the Gaussian random variable $\mathcal{N} (\mathbf{x}, \sigma^2I)$. The smoothed classifier $g_k(\mathbf{x})$ predicts the $k$ labels with the largest probabilities $p_i$'s for the example $\mathbf{x}$. We adopt randomized smoothing because it is scalable to large-scale neural networks and applicable to any base classifier.

Our major theoretical result is a tight certified robustness bound for top-$k$ predictions when using randomized smoothing with Gaussian noise. Specifically, given an example $\mathbf{x}$, a label $l$ is verifiably among the top-$k$ labels predicted by the smoothed classifier $g_k(\mathbf{x} + \delta)$ when the $\ell_2$-norm of the adversarial perturbation $\delta$ is less than a threshold (called \emph{certified radius}).  
The  certified radius for top-1 predictions derived by~\citet{cohen2019certified} is a special case of our certified radius when $k=1$. As our results and proofs show, generalizing certified robustness from top-1 to top-$k$ predictions faces significant new challenges and requires new techniques. Our certified radius is the unique solution to an equation, which depends on $\sigma$, $p_l$, and the $k$ largest probabilities $p_i$'s (excluding $p_l$). However,  computing our certified radius in practice faces two challenges: 1) it is hard to exactly compute the probability $p_l$ and  the $k$ largest probabilities $p_i$'s, and 2) the equation about the certified radius does not have an analytical solution. 
To address the first challenge, we estimate \emph{simultaneous confidence intervals} of the label probabilities via the Clopper-Pearson method and \emph{Bonferroni correction} in statistics. To address the second challenge, we propose an algorithm to solve the equation to obtain a lower bound of the certified radius, where the lower bound can be tuned to be arbitrarily close to the true certified radius. We evaluate our method on CIFAR10~\citep{krizhevsky2009learning} and ImageNet~\citep{deng2009imagenet} datasets. For instance, on ImageNet, our method respectively achieves approximate certified top-$1$, top-$3$, and top-$5$ accuracies as 46.6\%,   57.8\%, and 62.8\% when the $\ell_2$-norms of the adversarial perturbations are less than 0.5 (127/255) and $\sigma=0.5$.

Our contributions are summarized as follows: 
\begin{packeditemize}
\item {\bf Theory.} We derive the first certified radius for top-$k$ predictions. Moreover, we prove our certified radius is tight for randomized smoothing with Gaussian noise. 

\item {\bf Algorithm.} We develop algorithms to estimate our certified radius in practice. 

\item {\bf Evaluation.} We empirically evaluate our method on CIFAR10 and ImageNet. 
\end{packeditemize}

%% file: topk.tex
\section{Certified Radius for Top-$k$ Predictions}
Suppose we have a base classifier $f$, which maps an example $\mathbf{x}\in \mathbb{R}^d$ to one of $c$ candidate labels $\{1,2,\cdots,c\}$. $f$ can be any classifier. Randomized smoothing~\citep{cohen2019certified} adds an isotropic Gaussian noise $\mathcal{N}(0,\sigma^2I)$ to an example $\mathbf{x}$. We denote $p_{i}$ as the probability that the base classifier $f$ predicts label $i$ when adding a random isotropic Gaussian noise $\epsilon$ to the example $\mathbf{x}$, i.e., $p_{i}=\text{Pr}(f(\mathbf{x}+\epsilon)=i)$, where $\epsilon\sim \mathcal{N}(0,\sigma^2I)$. The smoothed classifier $g_k(\mathbf{x})$ returns the set of $k$ labels with the largest probabilities $p_i$'s when  taking an example $\mathbf{x}$ as input. Our goal is to derive a certified radius $R_l$ such that we have $l\in g_k(\mathbf{x} + \delta)$ for all $||\delta||_2 < R_l$. Our main theoretical results are summarized in the following two theorems. 

\begin{restatable}[Certified Radius for Top-$k$ Predictions]{thm}{certifyiedradiustheorem}
\label{theorem_of_certified_radius}
Suppose we are given an example $\mathbf{x}$, an arbitrary base classifier $f$, $\epsilon\sim \mathcal{N}(0,\sigma^2I)$, a smoothed classifier $g$, an arbitrary label $l\in \{1,2,\cdots,c\}$, and $\underline{p_l}, \overline{p}_1, \cdots, \overline{p}_{l-1}, \overline{p}_{l+1}, \cdots, \overline{p}_c \in [0,1]$ that satisfy the following conditions: 
\begin{align}
\label{consistent_condition}
   \text{Pr}(f(\mathbf{x}+\epsilon)=l) \geq \underline{p_l}\text{ and }   \text{Pr}(f(\mathbf{x}+\epsilon)=i) \leq \overline{p}_i, \forall i\neq l,
\end{align}
where $\underline{p}$ and $\overline{p}$ indicate lower and upper bounds of $p$, respectively. 
Let $\overline{p}_{b_k}\geq\overline{p}_{b_{k-1}}\geq\cdots\geq\overline{p}_{b_1}$  be the $k$ largest ones among $\{\overline{p}_1, \cdots, \overline{p}_{l-1}, \overline{p}_{l+1}, \cdots, \overline{p}_c\}$, where ties are broken uniformly at random.  Moreover, we denote by $S_t = \{b_1,b_2,\cdots,b_t\}$  the set of $t$ labels with the smallest probability upper bounds in the $k$ largest ones and by $\overline{p}_{S_{t}} =\sum_{j=1}^{t}\overline{p}_{b_j}$ the sum of the $t$ probability upper bounds, where $t=1,2,\cdots,k$. Then, we have: 
\begin{align}
    l \in g_{k}(\mathbf{x}+\delta),\forall ||\delta||_2 <R_l,
\end{align}
where $R_l$ is the unique solution to the following equation: 
\begin{align}
\label{equation_to_solve_for_topk}
\Phi(\Phi^{-1}(\underline{p_{l}})-\frac{R_l}{\sigma}))-\min_{t=1}^k \frac{\Phi(\Phi^{-1}(\overline{p}_{S_{t}})+\frac{R_l}{\sigma}))}{t}=0,  
\end{align}
where $\Phi$ and $\Phi^{-1}$ are the cumulative distribution function and its inverse of the standard Gaussian distribution, respectively. 
\end{restatable}
\begin{proof}
See Appendix~\ref{proof_theorem_of_certified_radius}.
\end{proof}

\begin{restatable}[Tightness of the Certified Radius]{thm}{tightofthebound}
\label{theorem_of_tight_of_the_bound}
Assuming we have  $\underline{p_{l}}+\sum_{j=1}^{k}\overline{p}_{b_j}\leq 1$ and $\underline{p_{l}}+\sum_{i=1,\cdots, l-1, l+1, \cdots, c}\overline{p}_{i}\geq 1$. 
Then, for any perturbation $||\delta||_2 >R_l$, there exists a base classifier $f^{*}$ consistent with~(\ref{consistent_condition}) but we have $l \notin g_{k}(\mathbf{x}+\delta)$.
\end{restatable}
\begin{proof} We show a proof sketch here. Our detailed proof is in Appendix~\ref{proof_of_tight_of_the_bound}. In our proof, we first show that, via \emph{mathematical induction} and the \emph{intermediate value theorem}, we can construct $k+1$ disjoint regions $\mathcal{C}_i, i\in \{l\}\cup\{b_1,b_{2},\cdots,b_k\}$ that satisfy $\text{Pr}(\mathbf{x}+\epsilon \in \mathcal{C}_l)=\underline{p_l}$, $\text{Pr}(\mathbf{x}+\epsilon \in \mathcal{C}_i)=\overline{p}_i$, and $\text{Pr}(\mathbf{x} + \delta + \epsilon \in \mathcal{C}_{i})$ is no smaller than some critical value for $i \in \{b_1,b_{2},\cdots,b_k\}$. Moreover, we divide the remaining region   $\mathbb{R}^d\setminus(\cup_{i=l,b_1,b_{2},\cdots,b_k} \mathcal{C}_{i})$ into $c-k-1$ regions, which we denote as $\mathcal{C}_{b_{k+1}}, \mathcal{C}_{b_{k+2}}, \cdots, \mathcal{C}_{b_{c-1}}$ and satisfy $\text{Pr}(\mathbf{x}+\epsilon \in \mathcal{C}_i)\leq \overline{p}_i$ for $i=b_{k+1}, b_{k+2}, \cdots, b_{c-1}$. Then, we construct a base classifier $f^{*}$ that predicts label $i$ for an example if and only if the example is in the region $\mathcal{C}_i$, where $i \in \{1, 2, \cdots, c\}$.  As $\underline{p_{l}}+\sum_{j=1}^{k}\overline{p}_{b_j}\leq 1$ and $\underline{p_{l}}+\sum_{i=1,\cdots, l-1, l+1, \cdots, c}\overline{p}_{i}\geq 1$, $f^{*}$ is well-defined. Moreover, $f^{*}$ satisfies the conditions in (\ref{consistent_condition}). Finally, we show that if $||\delta||_2 > R_l$, then we have  $\text{Pr}(f^{*}(\mathbf{x}+\delta + \epsilon)=l) < \min_{j=1}^k \text{Pr}(f^{*}(\mathbf{x}+\delta+\epsilon)=b_j)$, i.e., $l \notin g_{k}(\mathbf{x}+\delta)$.  
\end{proof}

We have several observations about our theorems. 
\begin{itemize}
\item Our certified radius is applicable to any base classifier $f$. 
\item According to Equation~\ref{equation_to_solve_for_topk}, our certified radius $R_l$ depends on $\sigma$, $\underline{p_l}$, and the $k$ largest probability upper bounds $\{\overline{p}_{b_k}, \overline{p}_{b_{k-1}}, \cdots, \overline{p}_{b_1}\}$ excluding  $\overline{p}_l$. When the lower bound $\underline{p_l}$ and the upper bounds $\{\overline{p}_{b_k}, \overline{p}_{b_{k-1}}, \cdots, \overline{p}_{b_1}\}$ are tighter, the certified radius $R_l$ is larger. When $R_l < 0$, the label $l$ is not among the top-$k$ labels predicted by the smoothed classifier even if no perturbation is added, i.e., $l\notin g_k(\mathbf{x})$. 
\item When using randomized smoothing with Gaussian noise and no further assumptions are made on the base classifier, it is impossible to certify a $\ell_2$ radius for top-$k$ predictions that is larger than $R_l$. 
\item When $k=1$, we have $R_l=\frac{\sigma}{2}(\Phi^{-1}(\underline{p_l}) - \Phi^{-1}(\overline{p}_{b_1}))$, where $\overline{p}_{b_1}$ is an upper bound of the  largest label probability excluding $\overline{p}_{l}$. The certified radius derived by~\citet{cohen2019certified} for top-1 predictions (i.e., their Equation 3) is a special case of our certified radius with $k=1$, $l=A$, and $b_1=B$.  
\end{itemize}

\begin{algorithm}[t]
\SetAlgoLined
\SetNoFillComment
\DontPrintSemicolon
\KwIn{$f$, $k$, $\sigma$, $\mathbf{x}$, $n$, and $\alpha$.}
\KwOut{ABSTAIN or predicted top-$k$ labels.}
$T$ $= \emptyset$ \\
counts $=  \textsc{SampleUnderNoise}(f,\sigma,\mathbf{x},n)$ \\
$c_{1},c_{2},\cdots,c_{k+1} = $ top-\{$k+1$\} indices in counts (ties are broken uniformly at random) \\
$n_{c_1},n_{c_2},\cdots,n_{c_{k+1}}=\text{counts}[c_{1}],\text{counts}[c_2],\cdots,\text{counts}[c_{k+1}]$ \\
\For{$t\gets1$ \KwTo $k$}{
\If{\textsc{BinomPValue}$(n_{c_t},n_{c_t}+n_{c_{t+1}},0.5)\leq \alpha$}{
   $T=T\cup c_t$ \\
}
\Else{
    \Return{\textup{ABSTAIN}}
}
}
\Return{$T$}
\caption{\textsc{Predict}}
\label{alg:predict}
\end{algorithm}

\section{Prediction and Certification in practice}
\subsection{Prediction}
It is challenging to compute the top-$k$ labels $g_k(\mathbf{x})$ predicted by the smoothed classifier, because it is challenging to compute the probabilities $p_i$'s exactly. To address the challenge, we resort to a Monte Carlo method that predicts the top-$k$ labels with a probabilistic guarantee. In particular, we leverage the hypothesis testing result from  a recent work~\citep{hung2019rank}. 
Algorithm~\ref{alg:predict} shows our \textsc{Predict} function to estimate the top-$k$ labels predicted by the smoothed classifier. The function \textsc{SampleUnderNoise}$(f,\sigma,\mathbf{x},n)$ first randomly samples $n$ noise $\epsilon_1, \epsilon_2, \cdots,\epsilon_n$ from the Gaussian distribution $\mathcal{N}(0,\sigma^2 I)$, uses the base classifier $f$ to predict the label of $\mathbf{x} + \epsilon_j$ for each $j\in \{1,2,\cdots,n\}$, and returns the frequency of each label, i.e.,   $\text{counts}[i]=\sum_{j=1}^{n}\mathbb{I}(f(\mathbf{x}+\epsilon_j)=i)$ for $i \in \{1,2,\cdots,c\}$. The function \textsc{BinomPValue} performs the hypothesis testing to calibrate the abstention threshold such that we can bound with probability $\alpha$ of returning an incorrect set of top-$k$ labels. Formally, we have the following proposition: 

\begin{proposition}
\label{proposition_1}
 With probability at least $1-\alpha$ over the randomness in \textsc{Predict}, if  \textsc{Predict} returns a set $T$ (i.e., does not ABSTAIN), then we have $g_k(\mathbf{x})=T$. 
\end{proposition}
\begin{proof}
See Appendix~\ref{proof_of_proposition_1}.
\end{proof}

\subsection{Certification}
Given a base classifier $f$, an example $\mathbf{x}$, a label $l$, and the standard deviation $\sigma$ of the Gaussian noise, we aim to compute the certified radius $R_l$. According to our Equation~\ref{equation_to_solve_for_topk}, our $R_l$ relies on a lower bound of $p_l$, i.e., $\underline{p_l}$, and the upper bound of $p_{S_t}$, i.e., $\overline{p}_{S_t}$, which are related to $f$, $\mathbf{x}$, and $\sigma$. We first discuss two Monte Carlo methods to estimate $\underline{p_l}$ and $\overline{p}_{S_t}$ with probabilistic guarantees. However, given     $\underline{p_l}$ and $\overline{p}_{S_t}$, it is still challenging to exactly solve $R_l$ as the Equation~\ref{equation_to_solve_for_topk} does not have an analytical solution. To address the challenge, we design an algorithm to obtain a lower bound of $R_l$ via solving Equation~\ref{equation_to_solve_for_topk} through binary search. Our lower bound can be tuned to be arbitrarily close to $R_l$.

\subsubsection{Estimating $\underline{p_l}$ and $\overline{p}_{S_t}$}
Our approach has two steps. The first step is to estimate $\underline{p_l}$ and $\overline{p}_i$ for $i\neq l$. The second step is to estimate $\overline{p}_{S_t}$ using $\overline{p}_i$ for $i\neq l$. 

\myparatight{Estimating $\underline{p_l}$ and $\overline{p}_i$ for $i\neq l$} The probabilities $p_1, p_2, \cdots, p_c$ can be viewed as a multinomial distribution over the labels $\{1,2,\cdots, c\}$. If we sample a Gaussian noise $\epsilon$ uniformly at random, then the label $f(\mathbf{x} +\epsilon)$ can be viewed as a sample from the multinomial distribution. Therefore,  estimating $\underline{p_l}$ and $\overline{p}_i$ for $i\neq l$ is essentially a one-sided \emph{simultaneous confidence interval} estimation problem. In particular, we aim to estimate these bounds with a confidence level at least $1-\alpha$. In statistics, ~\citet{goodman1965simultaneous,sison1995simultaneous} are well-known methods for simultaneous confidence interval estimations. However, these methods are insufficient for our problem. Specifically, Goodman's method is based on Chi-square test, which requires the expected count for each label to be no less than 5. We found that this is usually not satisfied, e.g., ImageNet has 1,000 labels, some of which have close-to-zero probabilities and do not have more than 5 counts even if we sample a large number of Gaussian noise. Sison \& Glaz's method guarantees a confidence level of \emph{approximately} $1-\alpha$, which means that the confidence level could be (slightly) smaller than $1-\alpha$. However, we aim to achieve a confidence level of at least $1-\alpha$. 
To address these challenges, we discuss two confidence interval estimation methods as follows:

{\bf 1) BinoCP.} This method estimates  $\underline{p_l}$ using the standard one-sided Clopper-Pearson method and treats $\overline{p}_i$ as $\overline{p}_i=1-\underline{p_l}$ for each $i\neq l$. Specifically, we sample $n$ random noise from $\mathcal{N}(0,\sigma I^2)$, i.e., $\epsilon_1, \epsilon_2, \cdots,\epsilon_n$. We denote the count for the label $l$  as $n_l = \sum_{j=1}^{n}\mathbb{I}(f(\mathbf{x}+\epsilon_j)=l)$. $n_l$ follows a binomial distribution with parameters $n$ and $p_l$, i.e., $n_l \sim Bin(n,p_l)$. Therefore, according to the Clopper-Pearson method, we have:
\begin{align}
     \underline{p_{l}}=B(\alpha; n_{l}, n-n_{l}+1),
\end{align}
where $1-\alpha$ is the confidence level and $B(\alpha; u,v)$ is the $\alpha$th quantile of the Beta distribution with shape parameters $u$ and $v$. We note that the Clopper-Pearson method was also adopted by~\citet{cohen2019certified} to estimate label probability for their certified radius of top-1 predictions.

{\bf 2) SimuEM.} The above method estimates $\overline{p}_{i}$ as $1-\underline{p_{l}}$, which may be conservative. A conservative estimation makes the certified radius smaller than what it should be. Therefore, we introduce SimuEM to directly estimate $\overline{p}_{i}$ together with $\underline{p_{l}}$. 
We let $n_i = \sum_{j=1}^n\mathbb{I}(f(\mathbf{x}+\epsilon_j)=i)$ for each $i \in \{1,2,\cdots,c\}$.  Each  $n_i$ follows a binomial distribution with parameters $n$ and $p_i$. We first use the Clopper-Pearson method to estimate a one-sided confidence interval for each label $i$, and then we obtain simultaneous confidence intervals by leveraging the \emph{Bonferroni correction}. Specifically, if we can obtain a confidence interval with confidence level at least $1-\frac{\alpha}{c}$ for each label $i$, then \emph{Bonferroni correction} tells us that the overall confidence level for the simultaneous confidence intervals is at least $1-\alpha$, i.e., we have confidence level at least $1-\alpha$ that all confidence intervals hold at the same time. Formally, we have the following bounds by applying the Clopper-Pearson method with confidence level $1-\frac{\alpha}{c}$ to each label:  
\begin{align}
\label{compute_lower_bound_simuem}
 & \underline{p_l }=B\left(\frac{\alpha}{c};n_l,n-n_l+1\right) \\
 \label{compute_upper_bound_simuem}
 &    \overline{p}_{i}=B(1-\frac{\alpha}{c};n_i+1, n-n_i), \  \forall i \neq l.
\end{align}

\begin{algorithm}[t]
\SetAlgoLined
\SetNoFillComment
\DontPrintSemicolon
\KwIn{$f$, $k$, $\sigma$, $\mathbf{x}$, $l$, $n$, $\mu$, and $\alpha$.}
\KwOut{ABSTAIN or $\underline{R_l}$.}
counts $=  \textsc{SampleUnderNoise}(f,\sigma,\mathbf{x},n,\alpha)$ \\
$[\underline{p_{l}}, \overline{p}_{1},\cdots,\overline{p}_{l-1}, \overline{p}_{l+1}, \cdots, \overline{p}_{c}]$ = \textsc{BinoCP}(counts, $\alpha$) \text{ or } \textsc{SimuEM}(counts, $\alpha$)\\
$\underline{R_l}=0$\\
\For{$t\gets1$ \KwTo $k$}{
    $\overline{p}_{S_t}=\min(\sum_{j=1}^t \overline{p}_{b_j}, 1-\underline{p_{l}})$\\
    \label{combine_to_estimate_p_sn}
    $\underline{R_l}^t=$\textsc{BinarySearch}( $\underline{p_l},\overline{p}_{S_t},t,\sigma,\mu$) \\
    \If{$\underline{R_l}^t > \underline{R_l}$}{
        $\underline{R_l} =\underline{R_l}^t$
    }
}
\If{$\underline{R_l} >0$}{
    \Return{$\underline{R_l}$} 
}
\Else{
    \Return{\textup{ABSTAIN}}
}
\caption{\textsc{Certify}}
\label{alg:certify}
\end{algorithm}

\myparatight{Estimating $\overline{p}_{S_t}$}  One natural method is to estimate $\overline{p}_{S_t}=\sum_{j=1}^{t}\overline{p}_{b_j}$. However, this bound may be loose. For example, when using \textbf{BinoCP} to estimate the probability bounds, we have $\overline{p}_{S_t}=t \cdot (1 - \underline{p_l})$, which may be bigger than $1$. To address the challenge, we derive another bound for $\overline{p}_{S_t}$ from another perspective. Specifically, we have ${p}_{S_t} \leq \sum_{i\neq l}p_i \leq 1 - \underline{p_l}$. Therefore, we can use $1 - \underline{p_l}$ as an upper bound of ${p}_{S_t}$, i.e., $\overline{p}_{S_t}=1 - \underline{p_l}$. Finally,  we  combine the above two estimations by taking the minimal one, i.e., $ \overline{p}_{S_t}=\min(\sum_{j=1}^{t}\overline{p}_{b_j},1 - \underline{p_l})$.

\subsubsection{Estimating a Lower Bound of the Certified Radius $R_l$}
It is challenging to compute the certified radius $R_l$ exactly because Equation~\ref{equation_to_solve_for_topk} does not have an analytical solution. To address the challenge, we design a method to estimate a lower bound of $R_l$ that can be tuned to be arbitrarily close to $R_l$. Specifically, we first approximately solve the following equation for each $t \in \{1,2,\cdots,k\}$:  
\begin{align}
\label{alternative_equation_to_solve_topk}
   \Phi(\Phi^{-1}(\underline{p_{l}})-\frac{R_l^t}{\sigma}))-\frac{\Phi(\Phi^{-1}(\overline{p}_{S_{t}})+\frac{R_l^t}{\sigma}))}{t} = 0. 
\end{align}
We note that it is still difficult to obtain an analytical solution to Equation \ref{alternative_equation_to_solve_topk} when $t>1$. However, we notice that the left-hand side has the following properties: 1) it decreases as $R_l^t$ increases; 2) when $R_l^t \rightarrow-\infty$, it is greater than 0; 3) when $R_l^t\rightarrow\infty$, it is smaller than 0. Therefore, there exists a unique solution $R_l^t$ to Equation \ref{alternative_equation_to_solve_topk}. Moreover, we leverage binary search to find a lower bound  $\underline{R_l}^{t}$ that can be arbitrarily close to the exact solution $R_l^t$. In particular, we run the binary search until the left-hand side of Equation \ref{alternative_equation_to_solve_topk} is non-negative and the width of the search interval is less than a parameter $\mu>0$. Formally, we have:  
\begin{align}
\label{binary_search_goal_in_r_t}
   \underline{R_l}^t \leq R_l^t \leq \underline{R_l}^t + \mu, \  \forall t\in \{1,2,\cdots,k\}. 
\end{align}
After obtaining $\underline{R_l}^t$, we let $\underline{R_l}=\max_{t=1}^k\underline{R_l}^t$ be our lower bound of $R_l$. Based on $R_l=\max_{t=1}^k R_l^t$ and  Equation \ref{binary_search_goal_in_r_t},  we have the following guarantee: 
\begin{align}
    \underline{R_l} \leq R_l  \leq \underline{R_l} + \mu. 
\end{align}

\subsubsection{Complete Certification Algorithm}

Algorithm~\ref{alg:certify} shows our algorithm to estimate the certified radius for a given example $\mathbf{x}$ and a label $l$. The function \textsc{SampleUnderNoise} is the same as in Algorithm~\ref{alg:predict}. Functions \textsc{BinoCP} and \textsc{SimuEM} return the estimated probability bound for each label. 
Function \textsc{BinarySearch} performs  binary search to solve the Equation~\ref{alternative_equation_to_solve_topk} and returns a solution satisfying Equation~\ref{binary_search_goal_in_r_t}. Formally, our algorithm has the following guarantee:

\begin{proposition}
\label{proposition_certify}
With probability at least $1 -  \alpha$ over the randomness in $\textsc{Certify}$, if $\textsc{Certify}$ returns a radius $\underline{R_l}$ (i.e., does not ABSTAIN), then we have     $l \in g_{k}(\mathbf{x}+\delta)$,  $\forall ||\delta||_2 < \underline{R_l}$. 
\end{proposition}

\begin{proof}
See Appendix~\ref{proof_of_proposition_2}. 
\end{proof}

%% file: exp.tex
\section{Experiments}

\subsection{Experimental Setup}

\myparatight{Datasets and models}
We conduct  experiments  on the standard CIFAR10 \citep{krizhevsky2009learning} and ImageNet \citep{deng2009imagenet} datasets to evaluate our method. We use the publicly available pre-trained models from \citet{cohen2019certified}. Specifically, the architectures of the base classifiers are ResNet-110 and ResNet-50 for CIFAR10 and ImageNet, respectively. 

\myparatight{Parameter setting} We study the impact of $k$, the confidence level $1-\alpha$, the noise level $\sigma$, the number of samples $n$, and the confidence interval estimation methods on the certified radius. Unless otherwise mentioned, we use the following default parameters: $k=3$, $\alpha = 0.001$, $\sigma = 0.5$, $n=100,000$, and $\mu=10^{-5}$. Moreover, we use $\textbf{SimuEM}$ to estimate bounds of label probabilities. When studying the impact of one parameter on the certified radius, we fix the other parameters to their default values. 

\myparatight{Approximate certified top-$k$ accuracy} For each testing example $\mathbf{x}$ whose true label is $l$, we compute the certified radius $\underline{R_l}$ using the \textsc{Certify} algorithm. Then, we compute the {certified top-$k$ accuracy} at a radius $r$  as the fraction of testing examples whose certified radius are at least $r$. Note that  our computed certified top-$k$ accuracy is an \emph{approximate certified top-$k$ accuracy} instead of the true certified top-$k$ accuracy. However, we can obtain a lower bound of the true certified top-$k$ accuracy based on the approximate certified top-$k$ accuracy. Appendix~\ref{certified_test_set_accuracy} shows the details.   Moreover, the gap between the lower bound of the true certified top-$k$ accuracy and the {approximate top-$k$ accuracy} is negligible when  $\alpha$ is small. For convenience, we  simply use the term certified top-$k$ accuracy in the paper.  

\begin{figure*}[!t]
	 \vspace{-13mm}
	 \centering
\subfloat[CIFAR10]{\includegraphics[width=0.40\textwidth]{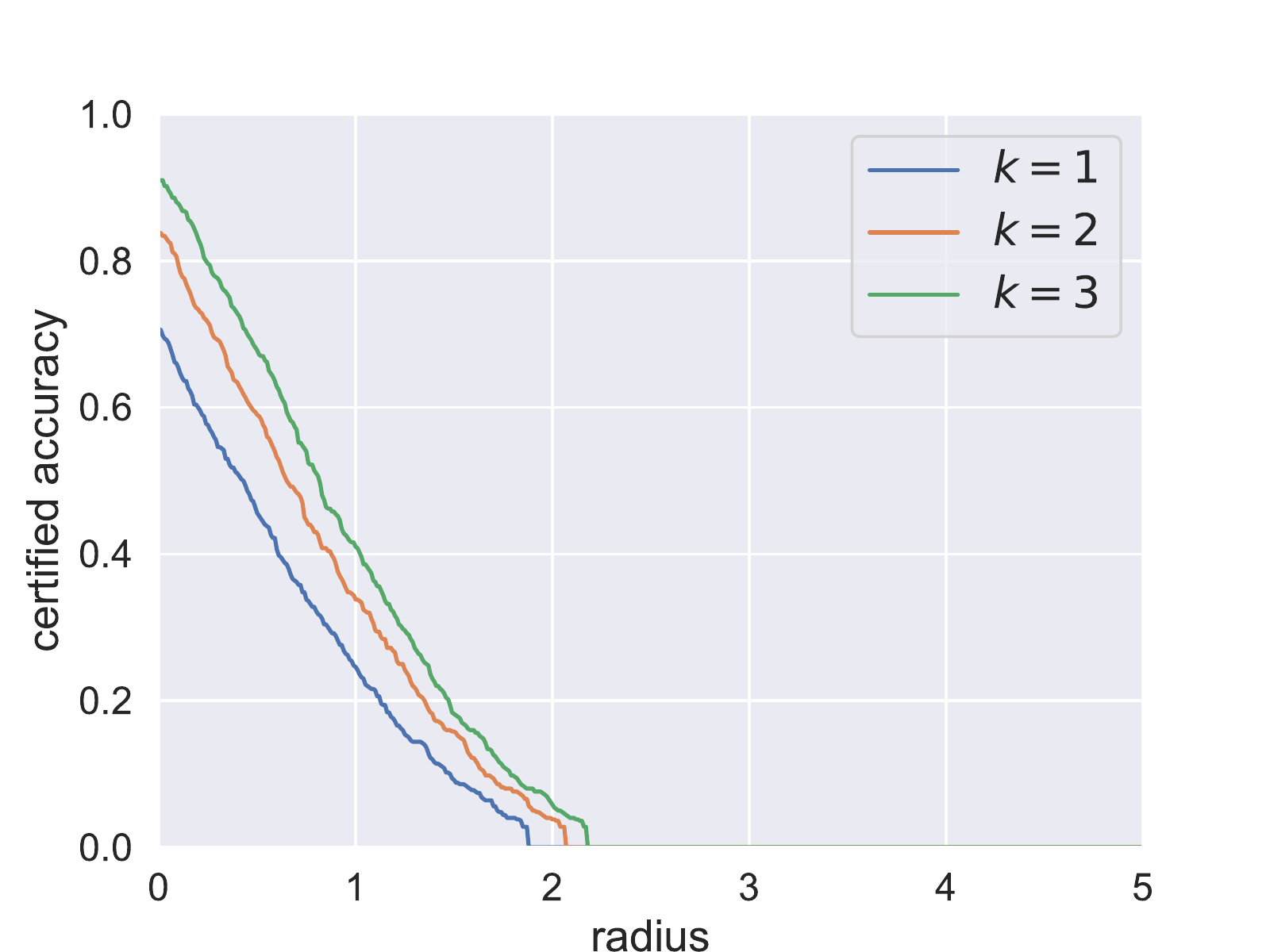}}
\subfloat[ImageNet]{\includegraphics[width=0.40\textwidth]{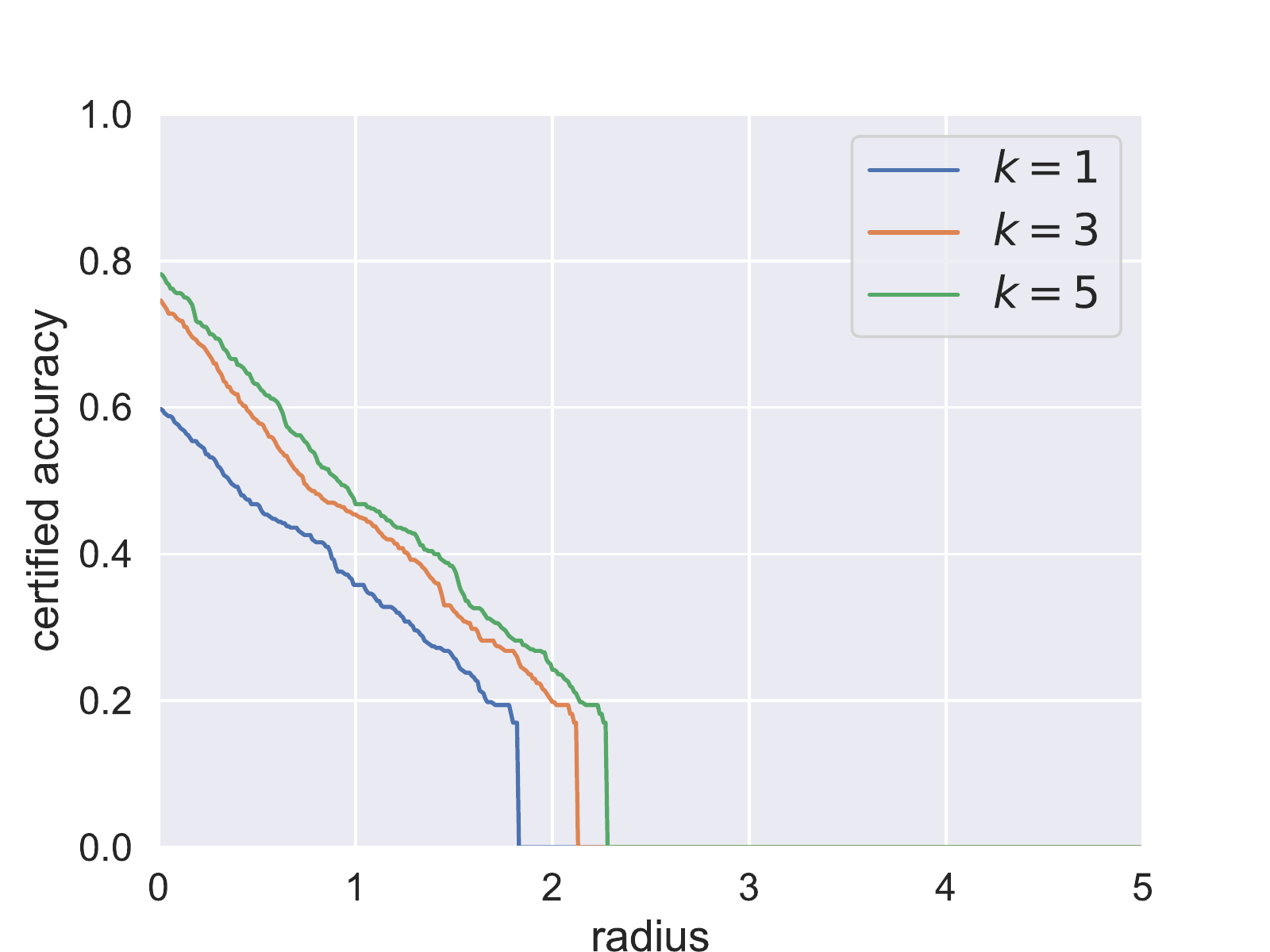}}
	 \caption{Impact of $k$ on the certified top-$k$ accuracy.}
	 \label{different_k_cifar10_imagenet}
\end{figure*}

\begin{figure*}[!t]
	 \vspace{-5mm}
	 \centering
	 \subfloat[CIFAR10]{\includegraphics[width=0.4\textwidth]{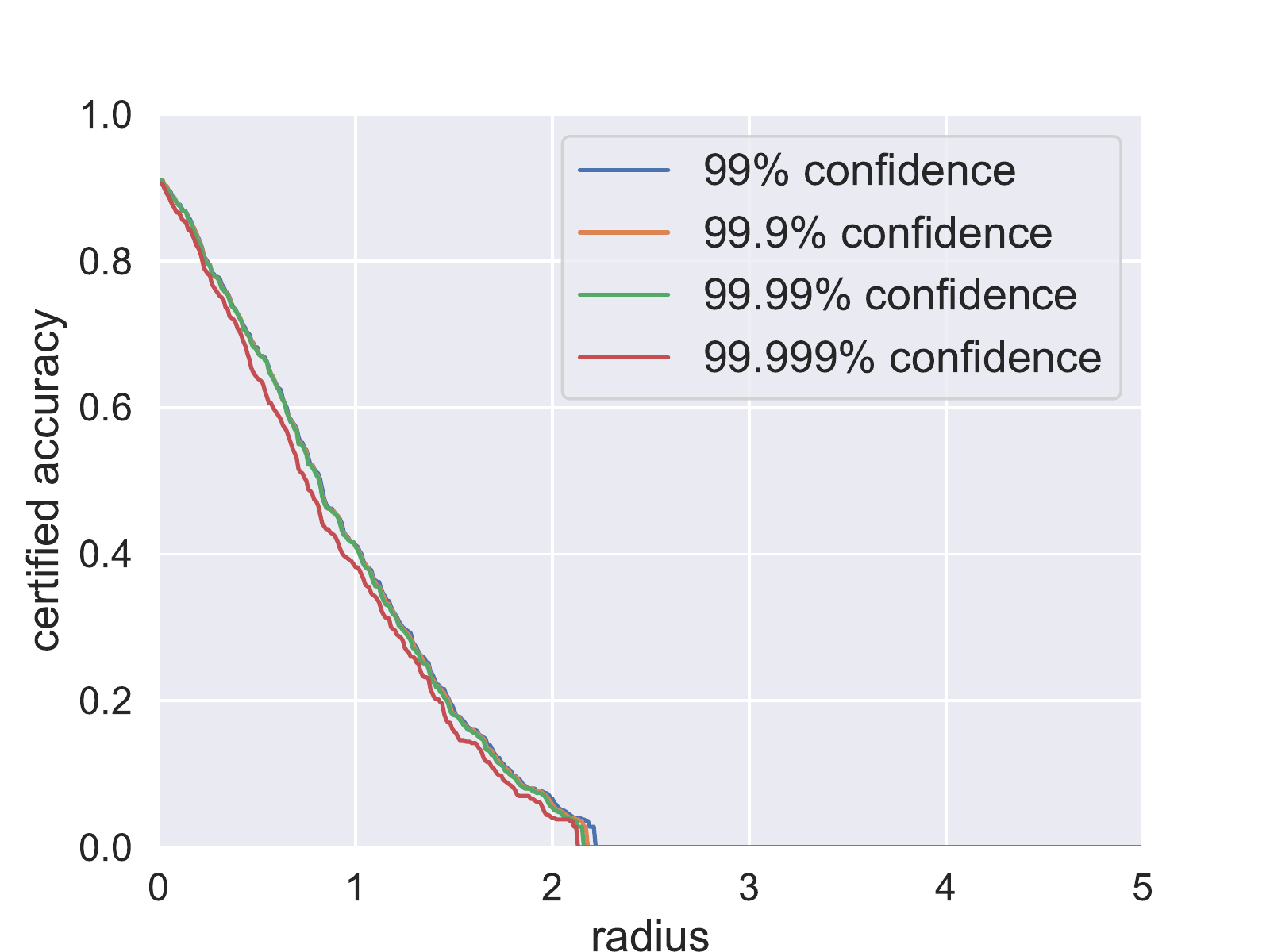}}
\subfloat[ImageNet]{\includegraphics[width=0.4\textwidth]{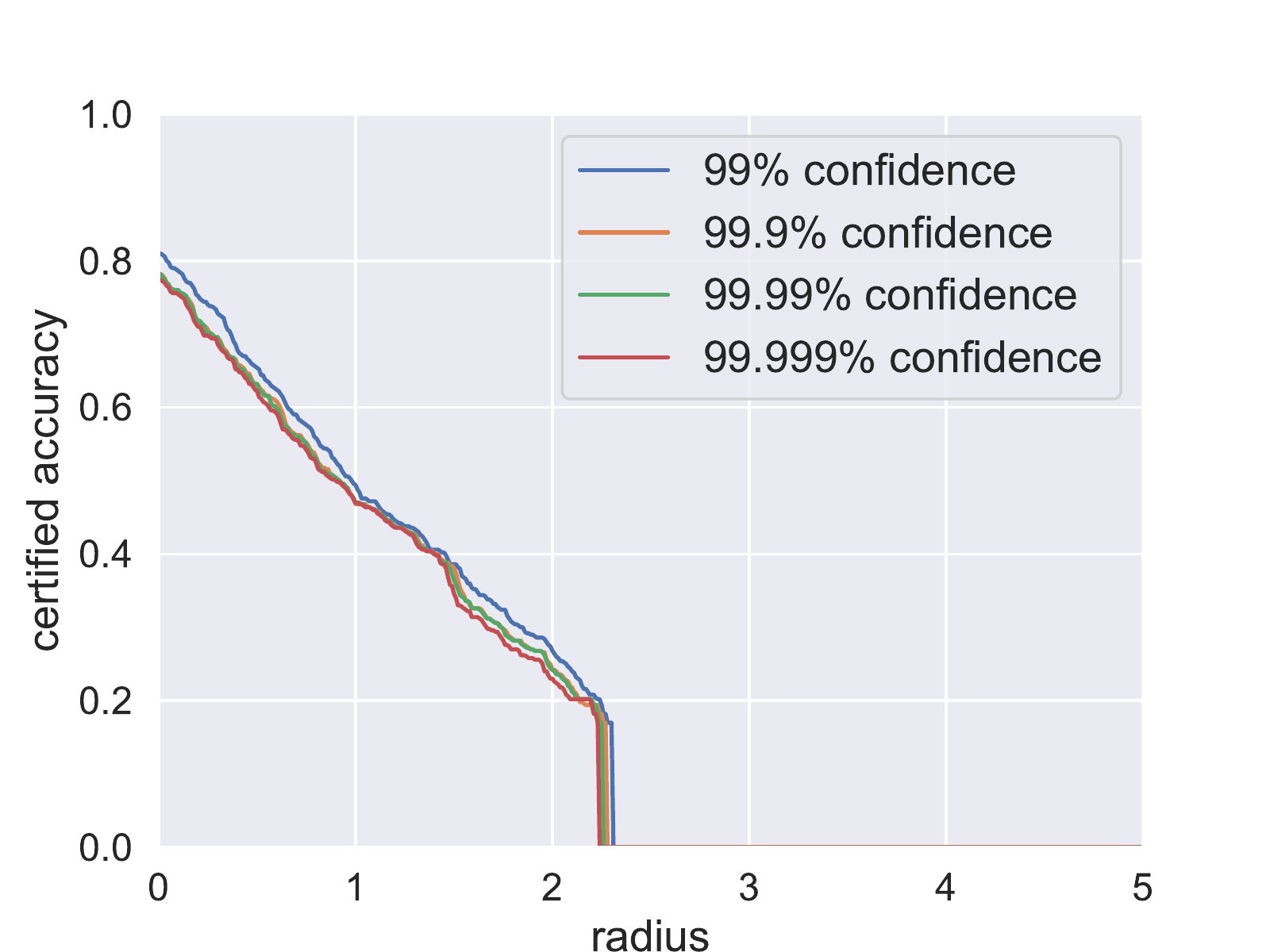}}
	 \caption{Impact of the confidence level $1-\alpha$ on the certified top-$3$ accuracy.}
	 \label{fig_compare_alpha_cifar10_imagenet}
	 \vspace{-8mm}
\end{figure*}

\begin{figure*}[!t]
	 \centering
\subfloat[CIFAR10]{\includegraphics[width=0.4\textwidth]{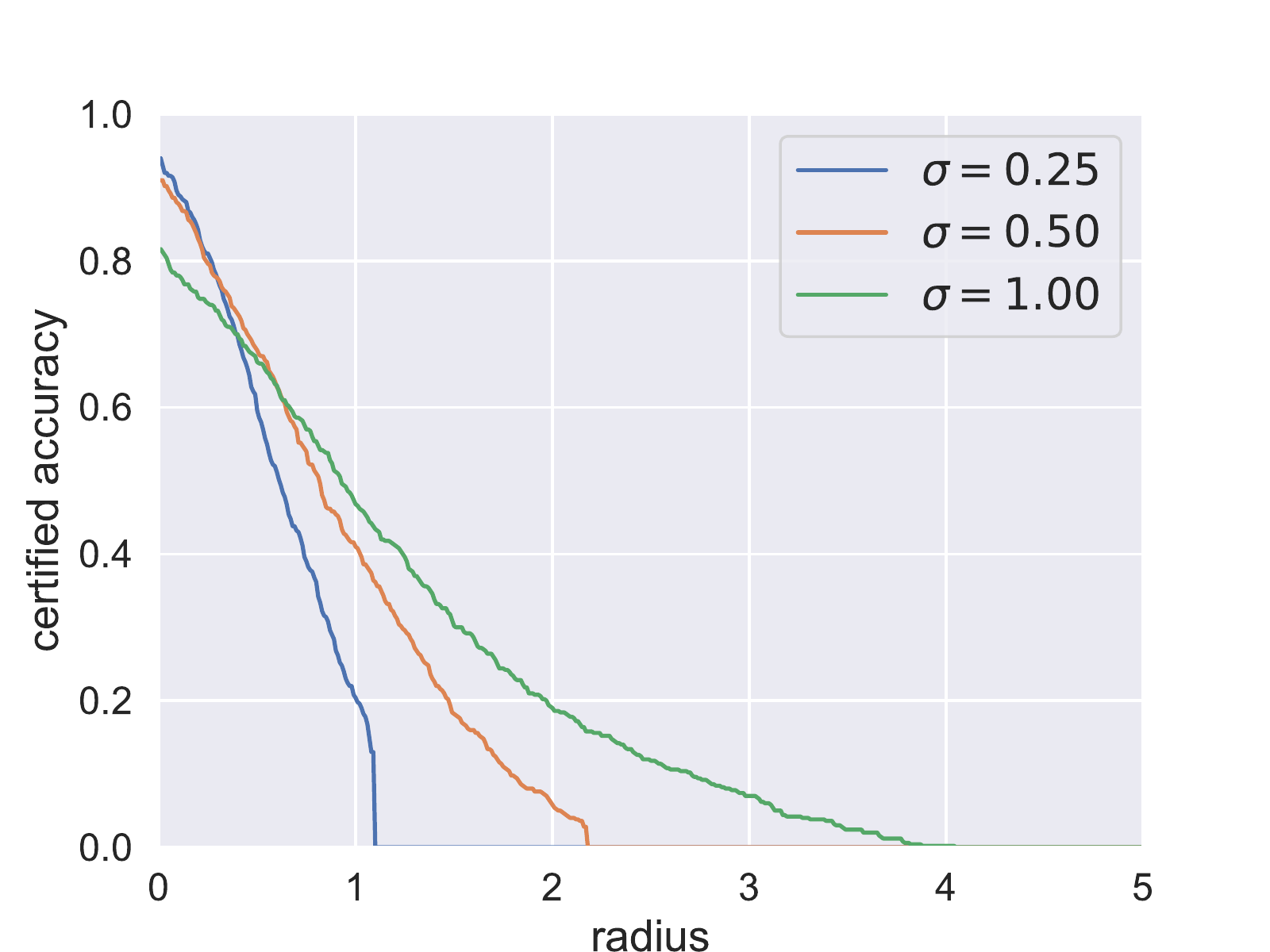}}
\subfloat[ImageNet]{\includegraphics[width=0.4\textwidth]{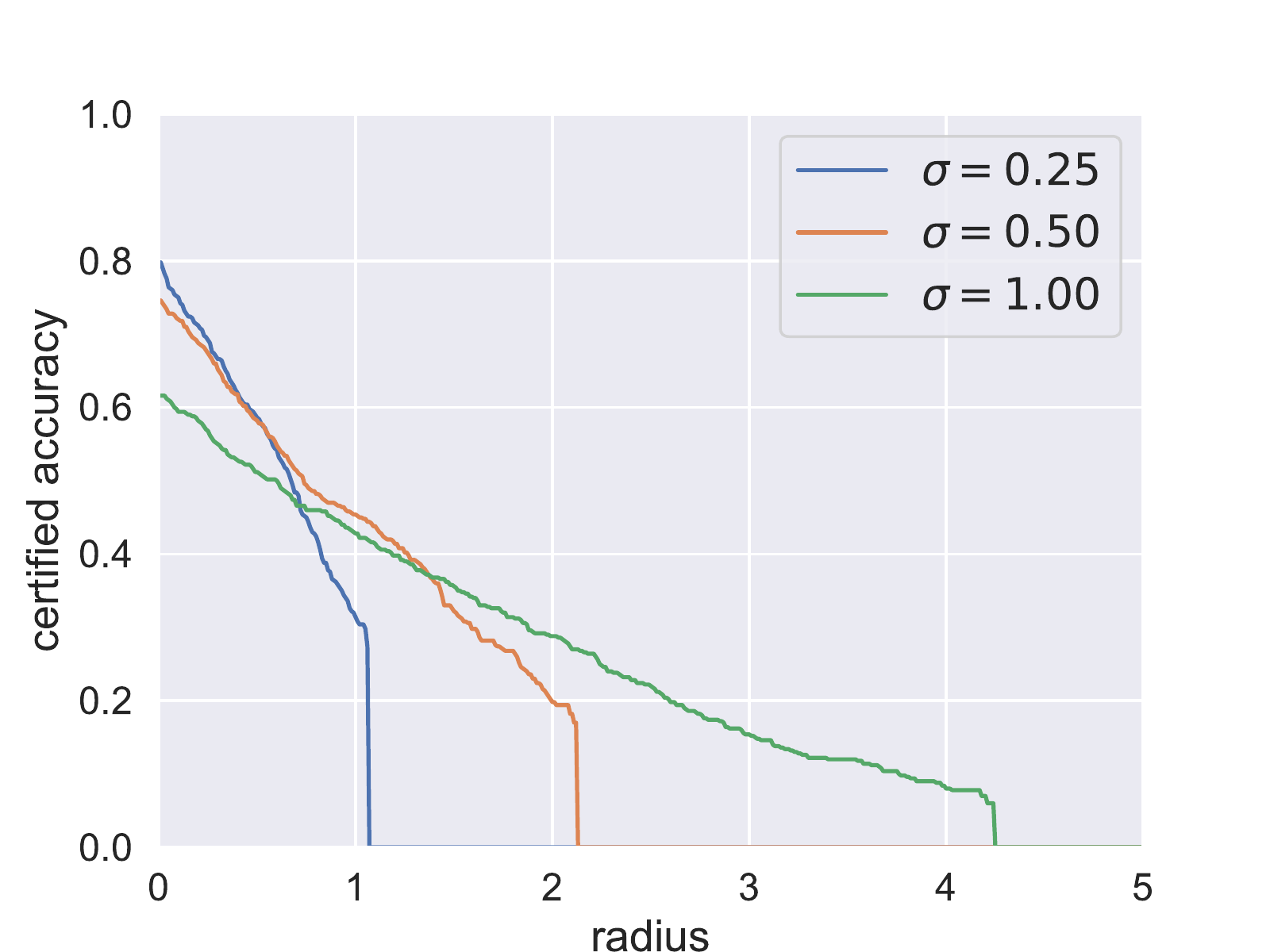}}
	 \caption{Impact of $\sigma$ on the certified top-$3$ accuracy.}
	 \label{different_sigma}
	 \vspace{-8mm}
\end{figure*}

\begin{figure*}[!t]
	 \centering
\subfloat[CIFAR10]{\includegraphics[width=0.4\textwidth]{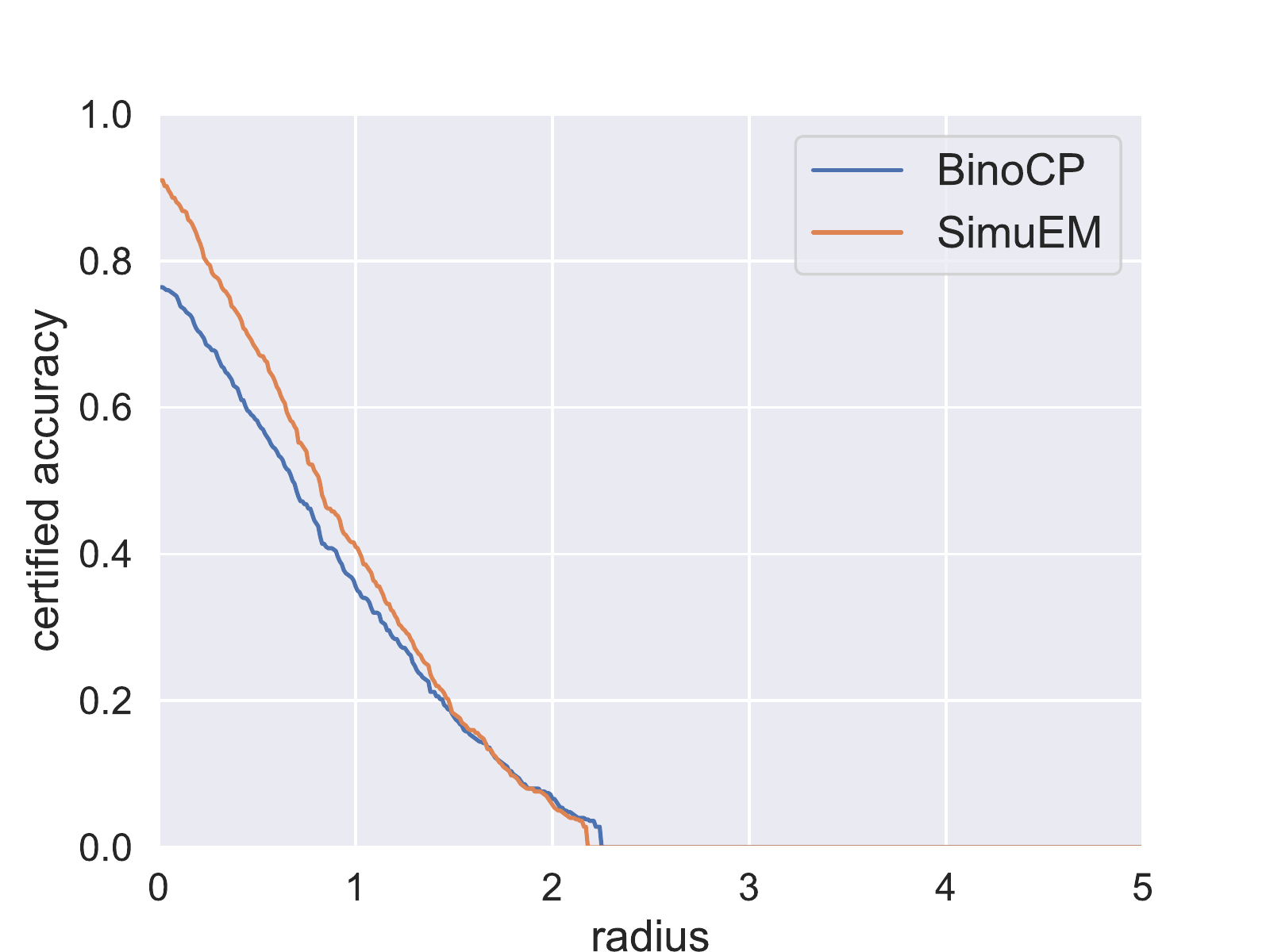}}
\subfloat[ImageNet]{\includegraphics[width=0.4\textwidth]{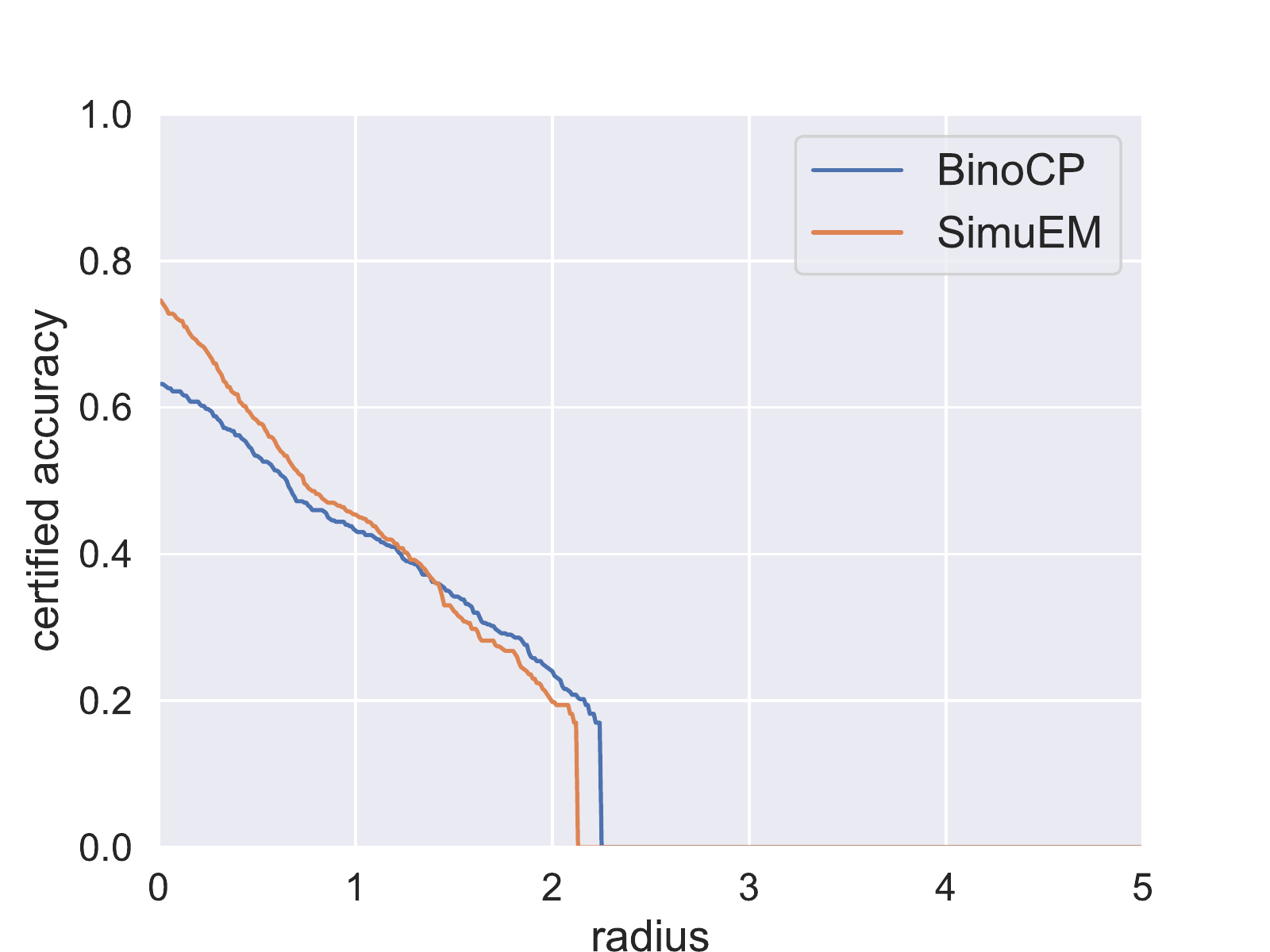}}
	 \caption{BinoCP vs. SimuEM, where $k=3$.}
	 \label{different_confidence_interval_estimation}
	 \vspace{-4mm}
\end{figure*}

\subsection{Experimental results}
Figure~\ref{different_k_cifar10_imagenet} shows the  certified top-$k$ accuracy as the radius $r$ increases for different $k$. Naturally, the  certified top-$k$ accuracy increases as  $k$ increases. On CIFAR10, we respectively achieve certified top-1, top-2, and top-3  accuracies as 45.2\%, 58.8\%, and 67.2\% when the $\ell_2$-norm of the adversarial perturbation is less than 0.5 (127/255). On ImageNet, we respectively achieve certified top-1, top-3, and top-5  accuracies as 46.6\%, 57.8\%, and 62.8\% when the $\ell_2$-norm of the adversarial perturbation is less than 0.5. 
On CIFAR10, the gaps between the certified top-$k$  accuracy for different $k$ are smaller than those between the top-$k$ accuracy under no attacks, and they become smaller as the radius increases. On ImageNet, the gaps between the certified top-$k$  accuracy for different $k$ remain similar to those between the top-$k$ accuracy under no attacks as the radius increases.
Figure~\ref{fig_compare_alpha_cifar10_imagenet} shows the influence of the confidence level. We observe that confidence level has a small influence on the certified top-$k$ accuracy as the different curves almost overlap. The reason is that the estimated confidence intervals of the probabilities shrink slowly as the confidence level increases. Figure~\ref{different_sigma} shows the influence of $\sigma$. We observe that $\sigma$ controls a trade-off between normal  accuracy under no attacks and robustness. Specifically, when $\sigma$ is smaller, the accuracy under no attacks (i.e., the accuracy when radius is 0) is larger, but the certified top-$k$ accuracy drops more quickly as the radius increases. 
Figure~\ref{different_confidence_interval_estimation} compares \textbf{BinoCP} with $\textbf{SimuEM}$. The results show that  $\textbf{SimuEM}$ is better when the certified radius is small, while \textbf{BinoCP} is better when the certified radius is large. We found the reason is that when the certified radius is large,  $\underline{p_l}$ is relatively large, and thus $1 - \underline{p_l}$ already provides a  good estimation for $\overline{p}_i$, where $i\neq l$.

%% file: related.tex
\section{Related Work}

Numerous defenses have been proposed against adversarial perturbations in the past several years. These defenses either show  robustness against existing attacks empirically, or prove the robustness against arbitrary bounded-perturbations (known as \emph{certified defenses}).   

\subsection{Empirical defenses}
The community has proposed many empirical defenses. The most effective empirical defense is \emph{adversarial training}~\citep{goodfellow2014explaining,Kurakin2017AdversarialML,tramer2018ensemble,madry2017towards}. However, adversarial training does not have certified robustness guarantees. Other examples of  empirical defenses include defensive distillation~\citep{papernot2016distillation}, MagNet~\citep{meng2017magnet}, PixelDefend~\citep{song2017pixeldefend}, Feature squeezing~\citep{xu2017feature}, and many others~\citep{liu2018adv,svoboda2018peernets,schott2019towards,buckman2018thermometer,ma2018characterizing,guo2018countering,dhillon2018stochastic,xie2018mitigating,song2018pixeldefend,samangouei2018defense,na2017cascade,metzen2017detecting}. However, many of these defenses were soon broken by adaptive attacks~\citep{carlini2017adversarial,athalye2018obfuscated,uesato2018adversarial,athalye2018robustness}.

\subsection{Certified defenses}
To end the arms race between defenders and adversaries, researchers have developed certified defenses against adversarial perturbations. Specifically, in a certifiably robust classifier, the predicted top-1 label is verifiably constant within a certain region (e.g., $\ell_2$-norm ball) around an input example, which provides a lower bound of the adversarial perturbation. Such certified defenses include satisfiability modulo theories based methods~\citep{katz2017reluplex,carlini2017provably,ehlers2017formal,huang2017safety}, mixed integer linear programming based methods~\citep{cheng2017maximum,lomuscio2017approach,dutta2017output,fischetti2018deep,bunel2018unified},  abstract interpretation based methods~\citep{gehr2018ai2,tjeng2017evaluating}, and global (or local) Lipschitz constant based methods~\citep{cisse2017parseval,gouk2018regularisation,tsuzuku2018lipschitz,anil2018sorting,wong2018provable,wang2018mixtrain,wang2018efficient,raghunathan2018certified,raghunathan2018semidefinite,wong2018scaling,dvijotham2018training,dvijotham2018dual,croce2018provable,gehr2018ai2,mirman2018differentiable,singh2018fast,gowal2018effectiveness,weng2018towards,zhang2018efficient}. However, these methods are not scalable to large neural networks and/or make assumptions on the architectures of the neural networks. For example, these defenses are not scalable/applicable to the complex neural networks for  ImageNet.

Randomized smoothing was first proposed as an empirical defense~\citep{cao2017mitigating,liu2018towards} without deriving the certified robustness guarantees. 
For instance, \citet{cao2017mitigating} proposed randomized smoothing with uniform noise from a hypercube centered at an example. 
\citet{lecuyer2018certified} was the first to prove the certified robustness guarantee of randomized smoothing for top-1 predictions. Their results leverage differential privacy. Subsequently, \citet{li2018second} further leverages information theory to improve the certified radius bound. \citet{cohen2019certified} obtains a tight certified radius bound for randomized smoothing with Gaussian noise by leveraging the Neyman-Pearson Lemma. \citet{pinot2019theoretical} theoretically demonstrated the robustness to adversarial attacks of randomized smoothing when adding noise from Exponential family distributions and devised an upper bound on the adversarial generalization gap of randomized neural networks. \citet{lee2019tight} generalized randomized smoothing to discrete data. \citet{salman2019provably} employed adversarial training to improve the performance of randomized smoothing. Unlike the other certified defenses, randomized smoothing is scalable to large neural networks and applicable to arbitrary classifiers. Our work derives the first certified robustness guarantee of randomized smoothing for top-$k$ predictions. Moreover, we show that our robustness guarantee is tight for randomized smoothing with Gaussian noise.

%% file: conclusion.tex
\section{Conclusion}
Adversarial perturbation poses a fundamental security threat to classifiers. Existing certified defenses focus on top-1 predictions, leaving top-$k$ predictions untouched. 
In this work, we derive the first certified radius under $\ell_2$-norm  for top-$k$ predictions. Our results are based on randomized smoothing. Moreover, we prove that our certified radius is tight for randomized smoothing with Gaussian noise.  In order to compute the certified radius in practice, we further propose simultaneous confidence interval estimation methods as well as design an algorithm to estimate a lower bound of the certified radius. 
Interesting directions for future work include 1) deriving a tight certified radius under other norms such as $\ell_1$ and $\ell_{\infty}$, 2) studying which noise gives the tightest certified radius for randomized smoothing, and 3) studying certified robustness for top-$k$ ranking.

%% file: appendix-new.tex
\appendix

\section{Proof of Theorem~\ref{theorem_of_certified_radius}}
\label{proof_theorem_of_certified_radius}

Given an example $\mathbf{x}$, we define the following two random variables: 
\begin{align}
   \label{definition_of_x}
   & \mathbf{X}=\mathbf{x}+\epsilon \sim \mathcal{N}(\mathbf{x},\sigma^{2}I), \\
     \label{definition_of_y}
   & \mathbf{Y}=\mathbf{x}+\delta+\epsilon \sim \mathcal{N}(\mathbf{x}+\delta,\sigma^{2}I),
\end{align}
where $\epsilon \sim \mathcal{N}(0,\sigma^2 I)$. The random variables $\mathbf{X}$ and $\mathbf{Y}$ represent random samples obtained by adding isotropic Gaussian noise to the example $\mathbf{x}$ and its perturbed version $\mathbf{x}+\delta$, respectively. 
\citet{cohen2019certified} applied the standard Neyman-Pearson Lemma~\citep{neyman1933ix} to the above two random variables, and obtained the following lemma: 
\begin{restatable}[Neyman-Pearson for Gaussians with different means]{lem}{lemmanpgaussian}
\label{lemma_np_gaussian}
Let $\mathbf{X}\sim \mathcal{N}(\mathbf{x},\sigma^2 I)$, $\mathbf{Y}\sim \mathcal{N}(\mathbf{x}+\delta,\sigma^2 I)$, and $M:\mathbb{R}^d\xrightarrow{} \{0,1\}$ be a random or deterministic function. Then, we have the following:  

(1) If $Z=\{\mathbf{z}\in \mathbb{R}^d: \delta^{T}\mathbf{z} \leq \beta  \}$ for some $\beta$ and $\text{Pr}(M(\mathbf{X})=1)\geq \text{Pr}(\mathbf{X}\in Z)$, then $\text{Pr}(M(\mathbf{Y})=1)\geq \text{Pr}(\mathbf{Y}\in Z)$

(2) If $Z=\{\mathbf{z}\in \mathbb{R}^d: \delta^{T}\mathbf{z} \geq \beta  \}$ for some $\beta$ and $\text{Pr}(M(\mathbf{X})=1)\leq \text{Pr}(\mathbf{X}\in Z)$, then $\text{Pr}(M(\mathbf{Y})=1)\leq \text{Pr}(\mathbf{Y}\in Z)$
\end{restatable}
Moreover, we have the following lemma from~\citet{cohen2019certified}. 
\begin{restatable}{lem}{lemmacomputealgebra}
\label{lemma_compute_algebra}
Given an example $\mathbf{x}$, a number $q\in [0,1]$, and regions $\mathcal{A}$ and $\mathcal{B}$ defined as follows: 
\begin{align}
    &\mathcal{A}=\{\mathbf{z}: \delta^{T}(\mathbf{z}-\mathbf{x})\leq \sigma\lnorm{\delta}_2\Phi^{-1}(q)\} \\
    &\mathcal{B}=\{\mathbf{z}: \delta^{T}(\mathbf{z}-\mathbf{x})\geq \sigma\lnorm{\delta}_2\Phi^{-1}(1-q)\} 
\end{align}
Then, we have the following equations: 
\begin{align}
&\text{Pr}(\mathbf{X}\in\mathcal{A})= q \\ &\text{Pr}(\mathbf{X}\in\mathcal{B})= q \\
&\text{Pr}(\mathbf{Y}\in \mathcal{A})=\Phi(\Phi^{-1}(q)-\frac{\lnorm{\delta}_2}{\sigma}) \\
&\text{Pr}(\mathbf{Y}\in \mathcal{B})=\Phi(\Phi^{-1}(q)+\frac{\lnorm{\delta}_2}{\sigma})
\end{align}
\end{restatable}
\begin{proof}
Please refer to \citet{cohen2019certified}. 
\end{proof}

Based on Lemma~\ref{lemma_np_gaussian} and~\ref{lemma_compute_algebra}, we derive the following lemma: 
\begin{restatable}{lem}{theoremcompare}
\label{theorem_compare}
Suppose we have an arbitrary base classifier $f$, an example $\mathbf{x}$, a set of labels which are denoted as $S$, two probabilities $\underline{p_S}$ and $\overline{p}_{S}$ that satisfy $\underline{p_S} \leq p_{S} =\text{Pr}(f(\mathbf{X})\in S)\leq \overline{p}_{S}$, and regions $\mathcal{A}_S$ and $\mathcal{B}_S$ defined as follows: 
\begin{align}
&\mathcal{A}_S=\{\mathbf{z}: \delta^{T}(\mathbf{z}-\mathbf{x})\leq \sigma\lnorm{\delta}_2\Phi^{-1}(\underline{p_{S}})\} \\
&\mathcal{B}_S=\{\mathbf{z}: \delta^{T}(\mathbf{z}-\mathbf{x})\geq \sigma\lnorm{\delta}_2\Phi^{-1}(1-\overline{p}_{S})\} 
\end{align}
Then, we have: 
\begin{align}
\label{equation_x_s_a}
&\text{Pr}(\mathbf{X}\in\mathcal{A}_S) \leq \text{Pr}(f(\mathbf{X})\in S)\leq \text{Pr}(\mathbf{X}\in\mathcal{B}_S)  \\
\label{equation_y_s_b}
&\text{Pr}(\mathbf{Y}\in \mathcal{A}_S) \leq \text{Pr}(f(\mathbf{Y})\in S)\leq \text{Pr}(\mathbf{Y}\in \mathcal{B}_S)
\end{align}
\end{restatable}
\begin{proof}
We know that $\text{Pr}(\mathbf{X}\in\mathcal{A}_S)=\underline{p_S}$ based on Lemma~\ref{lemma_compute_algebra}. Combined with the condition that $\underline{p_S} \leq \text{Pr}(f(\mathbf{X})\in S)$, we obtain the first inequality in~(\ref{equation_x_s_a}). Similarly, we can obtain the second inequality in~(\ref{equation_x_s_a}). We define $M(\mathbf{z})=\mathbb{I}(f(\mathbf{z})\in S)$. Based on the first inequality in~(\ref{equation_x_s_a}) and Lemma~\ref{lemma_np_gaussian}, we have the following:
\begin{align}
\text{Pr}(\mathbf{Y}\in \mathcal{A}_S)\leq\text{Pr}(M(\mathbf{Y})=1)=\text{Pr}(f(\mathbf{Y})\in S), 
\end{align}
which is the first inequality in~(\ref{equation_y_s_b}). The second inequality in~(\ref{equation_y_s_b}) can be obtained similarly.  
\end{proof}
Next, we {restate} Theorem~\ref{theorem_of_certified_radius} and show our proof. 

\certifyiedradiustheorem*
\begin{proof}
Roughly speaking, our idea is to make the probability that the base classifier $f$  predicts $l$ when taking $\mathbf{Y}$ as input larger than the smallest one among the probabilities that $f$ predicts for a set of arbitrary $k$ labels selected from all labels except $l$. 
For simplicity, we let $\Gamma=\{1,2,\cdots,{c}\}\setminus\{l\}$, i.e., all labels except $l$. We denote by $\Gamma_{k}$  a set of  $k$ labels in $\Gamma$. We aim to find a certified radius $R_l$ such that we have $\max_{\Gamma_{k} \subseteq \Gamma} \min_{i\in \Gamma_{k}}\text{Pr}(f(\mathbf{Y})=i) < \text{Pr}(f(\mathbf{Y})=l)$, which  guarantees $l \in g_{k}(\mathbf{x}+\delta)$. We first upper bound the minimal probability $\min_{i\in \Gamma_{k}}\text{Pr}(f(\mathbf{Y})=i)$ for a given $\Gamma_k$, and then we upper bound the maximum value of the minimal probability among all possible $\Gamma_k\subseteq\Gamma$. Finally, we obtain the certified radius $R_l$ via letting the upper bound of the maximum value smaller than $\text{Pr}(f(\mathbf{Y})=l)$.

\myparatight{Bounding $\min_{i\in \Gamma_{k}}\text{Pr}(f(\mathbf{Y})=i)$ for a given $\Gamma_k$} 
We use $S$ to denote a non-empty subset of $\Gamma_{k}$ and use $|S|$ to denote its size. We define $\overline{p}_{S} =\sum_{i\in S}\overline{p}_{i}$,  which is the sum of the upper bounds of the probabilities for the labels in $S$. Moreover, we define the following region associated with the set $S$:
\begin{align}
  \mathcal{B}_{S}=\{\mathbf{z}: \delta^{T}(\mathbf{z}-\mathbf{x})\geq \sigma\lnorm{\delta}_2\Phi^{-1}(1-\overline{p}_{S})\}  
\end{align}
We have $\text{Pr}(f(\mathbf{Y})\in S)\leq \text{Pr}(\mathbf{Y}\in\mathcal{B}_{S})$  by applying Lemma~\ref{theorem_compare} to the set $S$. In addition, we have $\sum_{i\in S}\text{Pr}(f(\mathbf{Y})=i) = \text{Pr}(f(\mathbf{Y})\in S)$. Therefore, we have: 
\begin{align}
\label{theorem_certified_radius_group}
    \sum_{i\in S}\text{Pr}(f(\mathbf{Y})=i) = \text{Pr}(f(\mathbf{Y})\in S)\leq \text{Pr}(\mathbf{Y}\in\mathcal{B}_{S})
\end{align}
Moreover, we have: 
\begin{align}
    \min_{i\in \Gamma_{k}}\text{Pr}(f(\mathbf{Y})=i)&\leq \min_{i\in S}\text{Pr}(f(\mathbf{Y})=i)\\
    									&\leq \frac{\sum_{i\in S}\text{Pr}(f(\mathbf{Y})=i)}{|S|} \\
									\label{minimal_value_condition}
									&\leq \frac{\text{Pr}(\mathbf{Y}\in\mathcal{B}_{S})}{|S|}, 
\end{align}
where we have the first inequality  because  $S$ is a subset of $\Gamma_{k}$ and we have the second inequality because the smallest value in a set is no larger than the average value of the set. 
 Equation~\ref{minimal_value_condition}  holds for any $S\subseteq \Gamma_{k}$. Therefore, by taking all possible sets $S$ into consideration, we have the following: 
\begin{align}
     \min_{i\in \Gamma_{k}}\text{Pr}(f(\mathbf{Y})=i)\leq & \min_{S\subseteq \Gamma_{k} }\frac{\text{Pr}(\mathbf{Y}\in\mathcal{B}_{S})}{|S|} \\
     \label{equationbound1}
     = & \min_{t=1}^k \min_{S\subseteq \Gamma_{k}, |S|=t}\frac{\text{Pr}(\mathbf{Y}\in\mathcal{B}_{S})}{t} \\
     \label{equationbound2}
     = &\min_{t=1}^k \frac{\text{Pr}(\mathbf{Y}\in\mathcal{B}_{S_t})}{t},
\end{align}
where $S_{t}$  is the set of $t$ labels in $\Gamma_k$ whose probability upper bounds are the smallest, where ties are broken uniformly at random. We have Equation~\ref{equationbound2} from Equation~\ref{equationbound1} because $\text{Pr}(\mathbf{Y}\in \mathcal{B}_S)$ decreases as  $\overline{p}_{S}$ decreases. 

\myparatight{Bounding $\max_{\Gamma_{k} \subseteq \Gamma} \min_{i\in \Gamma_{k}}\text{Pr}(f(\mathbf{Y})=i)$} Since $\text{Pr}(\mathbf{Y}\in\mathcal{B}_{S_t})$ increases as $\overline{p}_{S_t}$ increases, Equation~\ref{equationbound2} reaches its maximum value when $\Gamma_{k}=\{b_1,b_2,\cdots,b_k\}$, i.e., when $\Gamma_{k}$ is the set of $k$ labels in $\Gamma$ with the largest probability upper bounds. Formally, we have:
\begin{align}
\max_{\Gamma_{k} \subseteq \Gamma} \min_{i\in \Gamma_{k}}\text{Pr}(f(\mathbf{Y})=i) \leq \min_{t=1}^k \frac{\text{Pr}(\mathbf{Y}\in\mathcal{B}_{S_t})}{t},
\end{align}
where $S_t=\{b_1,b_2,\cdots,b_t\}$. 

\myparatight{Obtaining $R_l$} According to Lemma~\ref{theorem_compare}, we have the following for $S=\{l\}$: 
\begin{align}
    \text{Pr}(f(\mathbf{Y})=l)\geq \text{Pr}(\mathbf{Y}\in \mathcal{A}_{\{l\}})
\end{align}
Recall that our goal is to make $ \text{Pr}(f(\mathbf{Y})=l) > \max_{\Gamma_{k} \subseteq \Gamma} \min_{i\in \Gamma_{k}}\text{Pr}(f(\mathbf{Y})=i)$. It suffices to let: 
\begin{align}
   \text{Pr}(\mathbf{Y}\in\mathcal{A}_{\{l\}})>\min_{t=1}^k \frac{\text{Pr}(\mathbf{Y}\in\mathcal{B}_{S_t})}{t}. 
\end{align}
According to Lemma~\ref{lemma_compute_algebra}, we have $\text{Pr}(\mathbf{Y}\in\mathcal{A}_{\{l\}})=\Phi(\Phi^{-1}(\underline{p_{l}})-\frac{||\delta||_2}{\sigma}))$ and $\text{Pr}(\mathbf{Y}\in\mathcal{B}_{S_t})=\Phi(\Phi^{-1}(\overline{p}_{S_{t}})+\frac{||\delta||_2}{\sigma}))$. Therefore, we have the following constraint on $\delta$:
\begin{align}
\Phi(\Phi^{-1}(\underline{p_{l}})-\frac{||\delta||_2}{\sigma})) - \min_{t=1}^k \frac{\Phi(\Phi^{-1}(\overline{p}_{S_{t}})+\frac{||\delta||_2}{\sigma}))}{t}>0.
\end{align}
Since the left-hand side of the above inequality 1) decreases as $||\delta||_2$ increases, 2) is larger than 0 when $||\delta||_2\rightarrow -\infty$, and 3) is smaller than 0 when $||\delta||_2\rightarrow \infty$, we have the constraint $||\delta||_2 < R_l$, where $R_l$ is the unique solution to the following equation:
\begin{align}
\label{solutionRl}
\Phi(\Phi^{-1}(\underline{p_{l}})-\frac{R_l}{\sigma}))-\min_{t=1}^k \frac{\Phi(\Phi^{-1}(\overline{p}_{S_{t}})+\frac{R_l}{\sigma})) }{t}=0.
\end{align}
\end{proof}

\section{Proof of Theorem~\ref{theorem_of_tight_of_the_bound}}
\label{proof_of_tight_of_the_bound}
Following the terminology we used in proving Theorem~\ref{theorem_of_certified_radius}, we define a region $\mathcal{A}_{\{l\}}$ as follows:
\begin{align}
  \mathcal{A}_{\{l\}}=\{\mathbf{z}: \delta^{T}(\mathbf{z}-\mathbf{x})\leq \sigma\lnorm{\delta}_2\Phi^{-1}(\underline{p_{l}})\}. 
\end{align}
According to Lemma~\ref{lemma_compute_algebra}, we have $\text{Pr}(\mathbf{X}\in \mathcal{A}_{\{l\}})=\underline{p_{l}}$.
We first show the following lemma, which is the key to prove our Theorem~\ref{theorem_of_tight_of_the_bound}.
\begin{restatable}{lem}{lemmaofconditionoftightness}
\label{lemma_of_condition_of_tightness}
Assuming we have $\underline{p_l}+\sum_{j=1}^{k}\overline{p}_{b_j}\leq 1$. For any perturbation $\lnorm{\delta}_2 > R_l$, there exists $k$ disjoint regions $\mathcal{C}_{b_j}\subseteq \mathbb{R}^d\setminus\mathcal{A}_{\{l\}}$, $j\in\{1, 2, \cdots, k\}$ that satisfy the following: 
\begin{align}
\label{condition1}
   & \text{Pr}(\mathbf{X}\in \mathcal{C}_{b_j})=\overline{p}_{b_j},\ \forall j\in\{1, 2, \cdots, k\} \\
   \label{condition2}
    & \text{Pr}(\mathbf{Y}\in \mathcal{C}_{b_j})\geq \min_{t=1}^k\frac{\text{Pr}(\mathbf{Y}\in\mathcal{B}_{S_t})}{t}, \ \forall j\in\{1, 2, \cdots, k\}, 
\end{align}
where the random variables $\mathbf{X}$ and $\mathbf{Y}$ are defined in Equation~\ref{definition_of_x} and~\ref{definition_of_y}, respectively; and $\{b_1,b_2,\cdots,b_k\}$ and $S_t$ are defined in Theorem~\ref{theorem_of_certified_radius}.
\end{restatable}
\begin{proof}
Our proof is based on \emph{mathematical induction} and the \emph{intermediate value theorem}. For convenience, we defer the proof to  Appendix~\ref{proof_of_lemma_condition}. 
\end{proof}

Next, we {restate}  Theorem~\ref{theorem_of_tight_of_the_bound} and show our proof. 
\tightofthebound*
\begin{proof}

Our idea is to construct a base classifier such that  $l$ is not among the  top-$k$ labels predicted by the smoothed classifier for any perturbed example $\mathbf{x} + \delta$ when $||\delta||_2>R_l$. 
First, according to Lemma~\ref{lemma_of_condition_of_tightness}, we know there exists $k$ disjoint regions $\mathcal{C}_{b_j}\subseteq \mathbb{R}^d\setminus\mathcal{A}_{\{l\}}$, $j\in\{1, 2, \cdots, k\}$ that satisfy Equation~\ref{condition1} and~\ref{condition2}. Moreover, we divide the remaining region  $\mathbb{R}^d\setminus(\mathcal{A}_{\{l\}}\cup_{j=1}^k \mathcal{C}_{b_j})$ into $c-k-1$ regions, which we denote as $\mathcal{C}_{b_{k+1}}, \mathcal{C}_{b_{k+2}}, \cdots, \mathcal{C}_{b_{c-1}}$ and satisfy $\text{Pr}(\mathbf{X}\in \mathcal{C}_{b_j})\leq \overline{p}_{b_j}$ for $j\in\{k+1, k+2, \cdots, c-1\}$. Note that $b_1, b_2, \cdots, b_{c-1}$ is some permutation of $\{1,2,\cdots,c\}\setminus \{l\}$.  We can divide the remaining region into such $c-k-1$ regions because $\underline{p_{l}}+\sum_{i=1,\cdots, l-1, l+1, \cdots, c}\overline{p}_{i}\geq 1$. Then, based on these regions, we construct the following base classifier:
\begin{align}
    f^{*}(\mathbf{x})=
    \begin{cases}
     l, &\text{ if } \mathbf{x}\in \mathcal{A}_{\{l\}} \\
     b_j,  &\text{ if }  \mathbf{x}\in\mathcal{C}_{b_j}, \forall j\in \{1, 2,\cdots,c-1\} \\
    \end{cases}
\end{align}
Based on the definition of $f^*$, we have the following: 
\begin{align}
\text{Pr}(f^{*}(\mathbf{X})=l)&=\text{Pr}(\mathbf{X}\in \mathcal{A}_{\{l\}}) =  \underline{p_{l}} \\
\text{Pr}(f^{*}(\mathbf{X})=b_j)&=\text{Pr}(\mathbf{X}\in \mathcal{C}_{b_j}) =  \overline{p}_{b_j},\ j=1,2,\cdots,k \\
\text{Pr}(f^{*}(\mathbf{X})=b_j)&=\text{Pr}(\mathbf{X}\in \mathcal{C}_{b_j})\leq  \overline{p}_{b_j}, \  j=k+1,k+2,\cdots, c-1
\end{align}
Therefore, $f^*$ satisfies the conditions in~(\ref{consistent_condition}). Next, we show that $l$ is not among the top-$k$ labels predicted by the smoothed classifier for any perturbed example $\mathbf{x} + \delta$ when $||\delta||_2>R_l$. Specifically, we have:
\begin{align}
&\text{Pr}(f^{*}(\mathbf{Y})=l|\lnorm{\delta}_2 >R_l) \\
=&\text{Pr}(\mathbf{Y}\in\mathcal{A}_{\{l\}}|\lnorm{\delta}_2 >R_l) \hskip12mm (\text{Definition of } f^*)\\
<&\text{Pr}(\mathbf{Y}\in \mathcal{A}_{\{l\}}|\lnorm{\delta}_2 = R_l) \hskip12mm (\text{Pr}(\mathbf{Y}\in \mathcal{A}_{\{l\}})\text{ increases as } \lnorm{\delta}_2\text{ decreases})\\
=&\min_{t=1}^k \frac{\text{Pr}(\mathbf{Y}\in\mathcal{B}_{S_t}|\lnorm{\delta}_2 = R_l)}{t} \hskip6mm (\text{Definition of } R_l) \\
<&\min_{t=1}^k \frac{\text{Pr}(\mathbf{Y}\in\mathcal{B}_{S_t}|\lnorm{\delta}_2 >R_l)}{t} \hskip6mm (\text{Pr}(\mathbf{Y}\in\mathcal{B}_{S_t})  \text{ increases  as } \lnorm{\delta}_2 \text{ increases})\\
\leq&\text{Pr}(\mathbf{Y} \in \mathcal{C}_{b_j}|\lnorm{\delta}_2 >R_l)\hskip14mm (Lemma~\ref{lemma_of_condition_of_tightness})\\  
=&\text{Pr}(f^{*}(\mathbf{Y})=b_j|\lnorm{\delta}_2 >R_l) \hskip9mm (\text{Definition of } f^*),  
\end{align}
where $j=1,2,\cdots,k$. Since we have found $k$ labels whose probabilities are larger than the probability of the label $l$, we have $l \notin g_{k}(\mathbf{x}+\delta)$ when $\lnorm{\delta}_2 > R_l$. 
\end{proof}

\subsection{Proof of Lemma~\ref{lemma_of_condition_of_tightness}}
\label{proof_of_lemma_condition}
We first define some key notations and lemmas that will be used in our proof.
\begin{definition}[$\mathcal{C}(q_1,q_2)$, $\mathcal{C}^{\prime}(q_1,q_2)$, $r_{x}(q_1,q_2)$, $r_{y}(q_1,q_2)$]
\label{definition_for_the_region}
Given two values  $q_1$ and $q_2$ that satisfy $0\leq q_1 < q_2 \leq 1$, we define the following region:
\begin{align}
\mathcal{C}(q_1,q_2)=\{\mathbf{z}:\sigma\lnorm{\delta}_2\Phi^{-1}(1-q_2) < \delta^{T}(\mathbf{z} - \mathbf{x}) 
\leq \sigma\lnorm{\delta}_2\Phi^{-1}(1-q_1)\} 
\end{align}
According to Lemma~\ref{lemma_compute_algebra}, we have: 
\begin{align}
\label{equation_0_definition_of_the_region}
\text{Pr}(\mathbf{X}\in \mathcal{C}(q_1,q_2))=q_2-q_1, 
\end{align}
where the Gaussian random variable $\mathbf{X}$ is defined in Equation~\ref{definition_of_x}.
Moreover, assuming we have pairs of $(q_1^{i},q_2^{i})$, $i=1,2,3,\cdots$, where $q_1^i , q_2^i \in [q_1,q_2]$, $\forall i$. We define the following region: 
\begin{align}
    \mathcal{C}^{\prime}(q_1,q_2)=\mathcal{C}(q_1,q_2)\setminus(\cup_{i}\mathcal{C}(q_1^i,q_2^i)).
\end{align}
$\mathcal{C}^{\prime}(q_1,q_2)$ is the remaining region of $\mathcal{C}(q_1,q_2)$ excluding $\mathcal{C}(q_1^1,q_2^1), \mathcal{C}(q_1^2,q_2^2), \cdots$.
Given two values $q_1^{\lambda}$ and $q_2^{\lambda}$ that satisfy $q_1 \leq q_1^{\lambda}\leq q_2^{\lambda} \leq q_2$, we also define the following two functions:
\begin{align}
    &r_{x}(q_1^{\lambda},q_2^{\lambda})=\text{Pr}(\mathbf{X}\in \mathcal{C}^{\prime}(q_1,q_2)\cap \mathcal{C}(q_1^{\lambda},q_2^{\lambda})) \\
    &r_{y}(q_1^{\lambda},q_2^{\lambda})=\text{Pr}(\mathbf{Y}\in \mathcal{C}^{\prime}(q_1,q_2)\cap\mathcal{C}(q_1^{\lambda},q_2^{\lambda})),
\end{align}
where the random variables $\mathbf{X}$ and $\mathbf{Y}$ are defined in Equation~\ref{definition_of_x} and~\ref{definition_of_y}, respectively.
\end{definition}

Next, we show a key property of our defined functions $r_x(q_1^{\lambda},q_2^{\lambda})$ and $r_y(q_1^{\lambda},q_2^{\lambda})$. 
\begin{restatable}{lem}{lemmaq1q2}
\label{lemma_q1_q2}
If $r_x(q_1^{\kappa},q_2^{\kappa})\leq r_x(q_1^{\lambda},q_2^{\lambda})$ and $q_2^{\kappa}\geq q_2^{\lambda}$ (or $q_1^{\kappa}\geq q_1^{\lambda})$, then we have $r_y(q_1^{\kappa},q_2^{\kappa}) \leq r_y(q_1^{\lambda},q_2^{\lambda})$. 
\end{restatable}
\begin{proof}
We consider three scenarios. 

\myparatight{Scenario I}  $q_2^{\kappa}\geq q_1^{\kappa}\geq q_2^{\lambda}\geq q_1^{\lambda}$. We denote $h_x$ and $h_y$ as the probability densities for the random variables $\mathbf{X}$ and $\mathbf{Y}$, respectively. Then, we have  
$h_{x}(\mathbf{z})=(\frac{1}{\sqrt{2\pi}\sigma})^d\exp(-\frac{\sum_{i=1}^{d}(z_i - x_i)^2}{2\sigma^2})$ and $h_{y}(\mathbf{z})=(\frac{1}{\sqrt{2\pi}\sigma})^d\exp(-\frac{\sum_{i=1}^{d}(z_i - x_i - \delta_i)^2}{2\sigma^2})$. Therefore, the ratio of the probability density of $\mathbf{Y}$ and the probability density of $\mathbf{X}$ at a given point $\mathbf{z}$ is as follows: 
\begin{align}
\label{equation_0_lemma_q1_q2}
\frac{h_y(\mathbf{z})}{h_x(\mathbf{z})}=\exp(\frac{\delta^{T}(\mathbf{z}-\mathbf{x})}{\sigma^2}-\frac{\lnorm{\delta}_2}{2 \sigma^2})
\end{align}
Next, we compare the ratio for the points in different regions and have the following: 
\begin{align}
\label{equation_1_lemma_q1_q2}
& \frac{h_y(\mathbf{z}|\mathbf{z}\in \mathcal{C}^{\prime}(q_1^{\kappa},q_2^{\kappa}))}{h_x(\mathbf{z}|\mathbf{z}\in \mathcal{C}^{\prime}(q_1^{\kappa},q_2^{\kappa}))} \\
\label{equation_2_lemma_q1_q2}
\leq & \exp(\frac{\sigma\lnorm{\delta}_2\Phi^{-1}(1-q_1^{\kappa})}{\sigma^2}-\frac{\lnorm{\delta}_2}{2 \sigma^2}) \\
\label{equation_3_lemma_q1_q2}
\leq &
\exp(\frac{\sigma\lnorm{\delta}_2\Phi^{-1}(1-q_2^{\lambda})}{\sigma^2}-\frac{\lnorm{\delta}_2}{2 \sigma^2}) \\
\label{equation_4_lemma_q1_q2}
\leq & \frac{h_y(\mathbf{z}|\mathbf{z}\in \mathcal{C}^{\prime}(q_1^{\lambda},q_2^{\lambda}))}{h_x(\mathbf{z}|\mathbf{z}\in \mathcal{C}^{\prime}(q_1^{\prime},q_2^{\lambda}))}
\end{align}
The Equation~\ref{equation_2_lemma_q1_q2} from~\ref{equation_1_lemma_q1_q2} is based on Equation~\ref{equation_0_lemma_q1_q2} and the fact that $\delta^T(\mathbf{z}-\mathbf{x})\leq \sigma\lnorm{\delta}_2\Phi^{-1}(1-q_1^{\kappa})$ for any point $\mathbf{z}$ in the region $\mathcal{C}^{\prime}(q_1^{\kappa},q_2^{\kappa})$ from Definition~\ref{definition_for_the_region}. Similarly, we can obtain Equation~\ref{equation_4_lemma_q1_q2} from~\ref{equation_3_lemma_q1_q2}. We note that the Equation~\ref{equation_3_lemma_q1_q2} from~\ref{equation_2_lemma_q1_q2} is because $q_2^{\lambda}\leq q_1^{\kappa}$. Based on Equation~\ref{equation_2_lemma_q1_q2} and ~\ref{equation_3_lemma_q1_q2}, we know that there exists a real number $u$ such that:
\begin{align}
\label{equation_5_of_lemma_q1_q2}
 \exp(\frac{\sigma\lnorm{\delta}_2\Phi^{-1}(1-q_1^{\kappa})}{\sigma^2}-\frac{\lnorm{\delta}_2}{2 \sigma^2}) \leq u\leq \exp(\frac{\sigma\lnorm{\delta}_2\Phi^{-1}(1-q_2^{\lambda})}{\sigma^2}-\frac{\lnorm{\delta}_2}{2 \sigma^2})   
\end{align}
Combining the Equation~\ref{equation_1_lemma_q1_q2},~\ref{equation_2_lemma_q1_q2}, and~\ref{equation_5_of_lemma_q1_q2}, we have the following:
\begin{align}
\label{equation_6_lemma_q1_q2}
h_y(\mathbf{z}|\mathbf{z}\in \mathcal{C}^{\prime}(q_1^{\kappa},q_2^{\kappa})) \leq u\cdot h_x(\mathbf{z}|\mathbf{z}\in \mathcal{C}^{\prime}(q_1^{\kappa},q_2^{\kappa}))
\end{align}
Taking an integral on both sides of the Equation~\ref{equation_6_lemma_q1_q2} in the region $\mathcal{C}^{\prime}(q_1^{\kappa},q_2^{\kappa})$ and recalling the definition of $r_x(q_1^{\kappa},q_2^{\kappa})$ and $r_y(q_1^{\kappa},q_2^{\kappa})$, we have the following: 
\begin{align}
\label{lemma_q1_q2_equation_1}
    r_y(q_1^{\kappa},q_2^{\kappa})\leq u\cdot r_x(q_1^{\kappa},q_2^{\kappa})
\end{align}
Similarly, we have: 
\begin{align}
\label{lemma_q1_q2_equation_2}
u\cdot r_x(q_1^{\lambda},q_2^{\lambda})\leq r_y(q_1^{\lambda},q_2^{\lambda})
\end{align}
Based on  Equation~\ref{lemma_q1_q2_equation_1},~\ref{lemma_q1_q2_equation_2}, and the condition that $r_x(q_1^{\kappa},q_2^{\kappa}) \leq r_x(q_1^{\lambda},q_2^{\lambda})$, we have the following: 
\begin{align}
     r_y(q_1^{\kappa},q_2^{\kappa})\leq r_y(q_1^{\lambda},q_2^{\lambda})
\end{align}

\myparatight{Scenario II} $q_2^{\kappa} \geq q_2^{\lambda} \geq q_1^{\kappa} \geq q_1^{\lambda}$. 
We have: 
\begin{align}
    r_x(q_1^{\kappa},q_2^{\kappa})=r_x(q_1^{\kappa},q_2^{\lambda})+r_x(q_2^{\lambda},q_2^{\kappa})\leq r_x(q_1^{\lambda},q_1^{\kappa}) + r_x(q_1^{\kappa},q_2^{\lambda})=r_x(q_1^{\lambda},q_2^{\lambda})
\end{align}
Therefore, we have the following equation: 
\begin{align}
\label{lemma_q1_q2_equation_3}
r_x(q_2^{\lambda},q_2^{\kappa}) \leq r_x(q_1^{\lambda},q_1^{\kappa})
\end{align}
Similar to \textbf{Scenario I}, we know that there exists $u$ such that: 
\begin{align}
   & \frac{h_y(\mathbf{z}|\mathbf{z}\in \mathcal{C}^{\prime}(q_2^{\lambda},q_2^{\kappa}))}{h_x(\mathbf{z}|\mathbf{z}\in \mathcal{C}^{\prime}(q_2^{\lambda},q_2^{\kappa}))} \leq u \leq  \frac{h_y(\mathbf{z}|\mathbf{z}\in \mathcal{C}^{\prime}(q_1^{\lambda},q_1^{\kappa}))}{h_x(\mathbf{z}|\mathbf{z}\in \mathcal{C}^{\prime}(q_1^{\lambda},q_1^{\kappa}))}
\end{align}
Similar to \textbf{Scenario I}, we have the following based on Equation~\ref{lemma_q1_q2_equation_3}: 
\begin{align}
     r_y(q_2^{\lambda},q_2^{\kappa})\leq r_y(q_1^{\lambda},q_1^{\kappa})
\end{align}
Therefore, we have the following: 
\begin{align}
     r_y(q_1^{\kappa},q_2^{\kappa})=r_y(q_1^{\kappa},q_2^{\lambda})+r_y(q_2^{\lambda},q_2^{\kappa}) \leq r_y(q_1^{\lambda},q_1^{\kappa})+ r_y(q_1^{\kappa},q_2^{\lambda}) = r_y(q_1^{\lambda},q_2^{\lambda})
\end{align}

\myparatight{Scenario III} $q_2^{\kappa} \geq q_2^{\lambda} \geq q_1^{\lambda} \geq q_1^{\kappa}$. As   $r_x(q_1^{\kappa},q_2^{\kappa}) \leq r_x(q_1^{\lambda},q_2^{\lambda})$, we  have  $r_x(q_1^{\kappa},q_2^{\kappa}) = r_x(q_1^{\lambda},q_2^{\lambda})$. Therefore, we have  $r_y(q_1^{\kappa},q_2^{\kappa}) = r_y(q_1^{\lambda},q_2^{\lambda})$. 
\end{proof}

Next, we list the well-known  \emph{Intermediate Value Theorem} and show several other properties of our defined functions $r_x(q_1^{\lambda},q_2^{\lambda})$ and $r_y(q_1^{\lambda},q_2^{\lambda})$.

\begin{restatable}[Intermediate Value Theorem]{lem}{lemmaofintermediatevaluetheorem}
\label{lemma_of_intermediate_value_theorem}
If a function $F$ is continuous at every point in the interval $[a,b]$ and $(F(a)-v)\cdot (F(b)-v)\leq 0$, then there exists $x$ such that $F(x)=v$. 
\end{restatable}
Roughly speaking, the Intermediate Value Theorem tells us that if  a continuous function has values no larger and no smaller (or no smaller and no larger) than $v$ at the two end points of an interval, respectively, then  the function takes value $v$ at some point in the interval.

\begin{restatable}{lem}{lemmaofexist}
\label{lemma_of_exists}
 Given two probabilities $q_x, q_y$, if we have: 
\begin{align}
\label{lemma_of_exists_equation_1}
q_x\leq r_x(q_1 , q_2)
\end{align}
Then, there exists $q_1^{\prime}, q_2^{\prime}\in[q_1,q_2]$ such that: 
\begin{align}
    & r_x(q_1,q_2^{\prime})=q_x \\
    & r_x(q_1^{\prime},q_2)=q_x
\end{align}
Furthermore, if we have: 
\begin{align}
\label{equation_2_of_lemma_of_exists}
(r_y(q_1,q_2^{\prime})-q_y)\cdot ( r_y(q_1^{\prime},q_2)-q_y)\leq 0,
\end{align}
then there exists $q_1^{\prime\prime}, q_2^{\prime\prime}\in[q_1,q_2]$ such that:
\begin{align}
  &  r_x(q_1^{\prime\prime},q_2^{\prime\prime})=q_x \\
  &  r_y(q_1^{\prime\prime},q_2^{\prime\prime})=q_y
\end{align}
\end{restatable}
\begin{proof}
We define function $F(x)=r_x(q_1,x)$. Then, we have  $(F(q_1)-q_x)\cdot (F(q_2)-q_x)\leq 0$ since $F(q_1)=0 \leq q_x$ and $F(q_2)=r_x(q_1,q_2)>q_x$ based on  Equation~\ref{lemma_of_exists_equation_1}. Therefore, according to  Lemma~\ref{lemma_of_intermediate_value_theorem},  there exists $q_1^{\prime}\in [q_1,q_2]$ such that: 
\begin{align}
r_x(q_1,q_2^{\prime})=q_x
\end{align}
Similarly, we can prove that there exists $q_2^{\prime}\in [q_1,q_2]$ such that $r_x(q_1^{\prime},q_2)=q_x$. 

For any $q_2^e\in [q_2^{\prime},q_2]$, we define $H(x)=r_x(x,q_2^e)$. Then, we know  $H(q_1)=r_x(q_1,q_2^e)\geq r_x(q_1,q_2^{\prime})=q_x$ since $q_2^e \geq q_2^{\prime}$. Moreover, we have  $H(q_2^e)=0\leq q_x$. Therefore, we have  $(H(q_1)-q_x)\cdot (H(q_2^e)-q_x)\leq 0$. According to Lemma~\ref{lemma_of_intermediate_value_theorem},we know  there exists $q_1^e \in [q_1,q_2^e]$ such that $r_x(q_1^e,q_2^e)=q_x$ for arbitrary $q_2^e\in [q_2^{\prime},q_2]$. We define $G(x)=r_y(q_1^e,x)$ where $x\in [q_2^{\prime},q_2]$, and $q_1^e$ are a value such that $r_x(q_1^e,x)=q_x$ for a given $x$. When $x=q_2^{\prime}$, we can let $q_1^e=q_1$ since $r_x(q_1,q_2^{\prime})=q_x$, and when $x=q_1$, we can let $q_1^e=q_2^{\prime}$ since $r_x(q_2^{\prime},q_2)=q_x$. Based on  Equation~\ref{equation_2_of_lemma_of_exists} and Lemma~\ref{lemma_of_intermediate_value_theorem}, we know that there exists $x\in [q_2^{\prime},q_2]$ such that $G(x)=q_y$. Therefore, there exists $q_1^{\prime\prime}$ and $q_2^{\prime\prime}$ such that:
\begin{align}
  &  r_x(q_1^{\prime\prime},q_2^{\prime\prime})=q_x \\
  &  r_y(q_1^{\prime\prime},q_2^{\prime\prime})=q_y
\end{align}
\end{proof}

\begin{restatable}{lem}{lemma1lambda1q2}
\label{lemma_lambda_1_q_2}
Assuming we have  $q_1^{\lambda}=0$, $q_2^{\lambda} = \overline{p}_{S_t}$. If $q_2^{\lambda}=\overline{p}_{S_t} \leq \min_{i}q_1^i$, then we have the following: 
\begin{align}
\label{equation_o_of_lemma_lambda_1_q2}
&r_x(q_1^{\lambda},q_2^{\lambda})=q_2^{\lambda}-q_1^{\lambda}=\overline{p}_{S_t} \\
\label{equation_of_s_j}
&r_y(q_1^{\lambda},q_2^{\lambda}) =\text{Pr}(\mathbf{Y}\in\mathcal{B}_{S_t})
\end{align}
\end{restatable}
\begin{proof}
If $q_2^{\lambda} \leq \min_{i}q_1^i$, then we have  $\mathcal{C}^{\prime}(q_1,q_2)\cap\mathcal{C}(q_1^{\lambda},q_2^{\lambda})=\mathcal{C}(q_1^{\lambda},q_2^{\lambda})$ since no region is excluded. Therefore, we have $r_y(q_1^{\lambda},q_2^{\lambda})=\text{Pr}(\mathbf{Y}\in\mathcal{C}^{\prime}(q_1,q_2)\cap\mathcal{C}(q_1^{\lambda},q_2^{\lambda}))=\text{Pr}(\mathbf{Y}\in\mathcal{C}(q_1^{\lambda},q_2^{\lambda}))=q_2^{\lambda}-q_1^{\lambda}=\overline{p}_{S_t}$ based on  Equation~\ref{equation_0_definition_of_the_region}. 
We note that $\mathcal{C}(q_1^{\lambda},q_2^{\lambda})=\mathcal{B}_{S_t}$ when $q_1^{\lambda}=0$ and $q_2^{\lambda}=\overline{p}_{S_t}$. Therefore, we can obtain Equation~\ref{equation_of_s_j} based on the definition of $r_y(q_1^{\lambda},q_2^{\lambda})$ from Definition~\ref{definition_for_the_region}. 
\end{proof}

\begin{restatable}{lem}{lemmaofcomparemiddlevalue}
\label{lemma_of_compare_middle_value}
If we have $q_1\leq q_2^{o} \leq q_2$, then we have the following: 
\begin{align}
\label{inequality_from_lemma_q1_q2}
    r_y(q_1,q_2^{o})\geq\frac{r_y(q_1,q_2)}{\lceil r_x(q_1,q_2)/r_x(q_1,q_2^{o})\rceil}
\end{align}
\end{restatable}
\begin{proof}
 By applying Lemma~\ref{lemma_q1_q2}. 
\end{proof}

We  further generalize Lemma~\ref{lemma_q1_q2} to two regions. Specifically, we have the following lemma: 
\begin{restatable}{lem}{lemmaofcomparedisjoint}
\label{lemma_of_compare_disjoint}
Assuming we have a region $\mathcal{C}_w\subseteq\mathcal{C}(q_1^w,q_2^w)$ and we have $\mathcal{C}^{\prime}(q_1,q_2)\cap\mathcal{C}(q_1^w,q_2^w)=\emptyset$. If $q_1\geq q_1^w$, $q_2\geq q_2^w$ and $r_x(q_1^{\prime},q_2^{\prime})\leq \text{Pr}(\mathbf{X}\in \mathcal{C}_w)$, then we have: 
\begin{align}
    r_y(q_1,q_2)\leq \text{Pr}(\mathbf{Y}\in \mathcal{C}_w)
\end{align}
\end{restatable}
\begin{proof}
We let $q_1=\max(q_1,q_2^w)$. As $\mathcal{C}^{\prime}(q_1,q_2)\cap\mathcal{C}(q_1^w,q_2^w)=\emptyset$ and $q_1 \geq q_1^w$, we can obtain the conclusion 
 by applying Lemma~\ref{lemma_q1_q2} on $\mathcal{C}^{\prime}(q_1,q_2)\cup\mathcal{C}_w$. 
\end{proof}

Next, we restate Lemma~\ref{lemma_of_condition_of_tightness} and show our proof. 
\lemmaofconditionoftightness*
\begin{proof}
Our proof leverages \emph{Mathematical Induction}, which contains two steps. In the first step, we  show that the statement holds initially. In the second step, we  show that if the statement is true for the $m$th iteration, then it also holds for the $(m+1)$th iteration. 
Without loss of generality, we assume $\tau=\argmin_{t=1}^{k} \frac{\text{Pr}(\mathbf{Y}\in\mathcal{B}_{S_t})}{t}$. Therefore, we have the following: 
\begin{align}
\label{compare_others_with_j}
    \forall i \neq \tau, \frac{\text{Pr}(\mathbf{Y}\in\mathcal{B}_{S_i})}{i}\geq\frac{\text{Pr}(\mathbf{Y}\in\mathcal{B}_{S_{\tau}})}{\tau}
\end{align} 
Recall the definition of $\mathcal{B}_{S_{\tau}}$ and we have the following:
\begin{align}
  \mathcal{B}_{S_{\tau}}=\{\mathbf{z}: \delta^{T}(\mathbf{z}-\mathbf{x})\geq \sigma\lnorm{\delta}_2\Phi^{-1}(1-\overline{p}_{S_{\tau}})\}, 
\end{align}
where $\overline{p}_{S_{\tau}} =\sum_{j\in S_{\tau}}\overline{p}_{j}$.
We can split $\mathcal{B}_{S_k}$ into two parts: $\mathcal{B}_{S_{\tau}}$ and $\mathcal{B}_{S_k}\setminus \mathcal{B}_{S_{\tau}}$. 
We will show that $\forall j\in [1,\tau]$, we can find disjoint $\mathcal{C}_{b_j}\subseteq \mathcal{B}_{S_{\tau}}$ whose union is $\mathcal{B}_{S_{\tau}}$ such that: 
\begin{align}
\label{construction_goal_1_part1}
 &   \text{Pr}(\mathbf{X}\in \mathcal{C}_{b_j})=\overline{p}_{b_j} \\
 \label{construction_goal_2_part1}
 & \text{Pr}(\mathbf{Y}\in \mathcal{C}_{b_j})=\frac{\text{Pr}(\mathbf{Y}\in\mathcal{B}_{S_{\tau}})}{\tau}
\end{align}
For the other part, we will show that $\forall j\in [\tau+1,k]$, we can find disjoint $\mathcal{C}_{b_j}\subseteq \mathcal{B}_{S_k}\setminus\mathcal{B}_{S_{\tau}}$ whose union is $\mathcal{B}_{S_k}\setminus\mathcal{B}_{S_{\tau}}$ such that: 
\begin{align}
 &   \text{Pr}(\mathbf{X}\in \mathcal{C}_{b_j})=\overline{p}_{b_j} \\
 & \text{Pr}(\mathbf{Y}\in \mathcal{C}_{b_j})\geq\frac{\text{Pr}(\mathbf{Y}\in\mathcal{B}_{S_{\tau}})}{\tau}
\end{align}

We  first show that $\forall j\in [1,\tau]$, we can find $\mathcal{C}_{b_j}\subseteq \mathcal{B}_{S_{\tau}}$ that satisfy Equation~\ref{construction_goal_1_part1} and~\ref{construction_goal_2_part1}. Since our proof leverages Mathematical Induction, we  iteratively construct each $\mathcal{C}_{b_j},\forall j\in [1,\tau]$. Specifically, we  first show that we can find $\mathcal{C}_{b_{\tau}} \subseteq \mathcal{B}_{S_{\tau}}$ that satisfies the requirements. Then, assuming we can find $\mathcal{C}_{b_{\tau}},\cdots,\mathcal{C}_{b_{\tau-m+1}}$, we show that we can find $\mathcal{C}_{b_{\tau-m}}\subseteq \mathcal{B}_{S_{\tau}}\setminus(\cup_{j=\tau-m+1}^{\tau}\mathcal{C}_{b_j})$. We will leverage Lemma~\ref{lemma_of_exists} to prove the existence for each $\mathcal{C}_{b_j}$. Next, we  show the two steps.

\myparatight{Step I}
We  show that we can find $\mathcal{C}_{b_{\tau}} \subseteq \mathcal{B}_{S_{\tau}}$ that satisfies Equation~\ref{construction_goal_1_part1} and~\ref{construction_goal_2_part1}. We let $q_1 =0$ and $q_2=\overline{p}_{S_{\tau}}$, and we define the following region: 
\begin{align}
    \mathcal{C}^{\prime}(q_1,q_2)= \mathcal{C}(0,\overline{p}_{S_{\tau}})=\mathcal{B}_{S_{\tau}}
\end{align}
We have:  
\begin{align}
\label{part1_stepI_rx_q1_q2_value_compute}
&r_x(q_1,q_2)= \text{Pr}(\mathbf{X}\in \mathcal{C}^{\prime}(q_1,q_2))=\overline{p}_{S_{\tau}} \\
\label{part1_stepI_ry_q1_q2_value_compute}
&r_y(q_1,q_2)= \text{Pr}(\mathbf{Y}\in \mathcal{C}^{\prime}(q_1,q_2))=\text{Pr}( \mathbf{Y}\in \mathcal{B}_{S_{\tau}}),
\end{align}
which can be directly obtained as $\mathcal{C}^{\prime}(q_1,q_2)=\mathcal{B}_{S_{\tau}}$. As we have $\overline{p}_{b_{\tau}}\leq r_x(q_1 , q_2)=\overline{p}_{S_{\tau}}=\sum_{j=1}^{\tau}\overline{p}_{b_j}$, there exist $q_1^{\prime}=\overline{p}_{S_{\tau}}-\overline{p}_{b_{\tau}},q_2^{\prime}=\overline{p}_{b_{\tau}}$ such that: 
\begin{align}
   & r_x(q_1,q_2^{\prime})=\overline{p}_{b_{\tau}} \\
    &r_x(q_1^{\prime},q_2)=\overline{p}_{b_{\tau}}
\end{align}
Moreover, we have the following:  
\begin{align}
    r_y(q_1,q_2^{\prime}) 
    \geq r_y(q_1,\overline{p}_{b_1}) 
    =\text{Pr}(\mathbf{Y}\in\mathcal{C}(q_1,\overline{p}_{b_1})) 
    =\text{Pr}(\mathbf{Y}\in\mathcal{B}_{S_1}) 
    \geq\frac{\text{Pr}(\mathbf{Y}\in \mathcal{B}_{S_{\tau}})}{\tau}
\end{align}
The equality in the middle is from Lemma~\ref{lemma_lambda_1_q_2},
the left inequality is because $q_2^{\prime}=\overline{p}_{b_{\tau}}\geq\overline{p}_{b_1}$, and the right inequality is from Equation~\ref{compare_others_with_j}. 
Furthermore, we have the following: 
\begin{align}
&r_y(q_1^{\prime},q_2) \\
= &\text{Pr}(\mathbf{Y}\in \mathcal{C}^{\prime}(q_1,q_2)\cap\mathcal{C}(q_1^{\prime},q_2)) \\
= &\text{Pr}(\mathbf{Y}\in \mathcal{C}(\overline{p}_{S_{\tau}} - \overline{p}_{b_{\tau}},\overline{p}_{S_{\tau}})) \\
= &\text{Pr}(\mathbf{Y}\in \mathcal{C}(0,\overline{p}_{S_{\tau}})) - \text{Pr}(\mathbf{Y}\in \mathcal{C}(0,\overline{p}_{S_{\tau}} - \overline{p}_{b_{\tau}})) \\
\label{part1_step1_r_y_q2_prime_1}
= &\text{Pr}(\mathbf{Y}\in \mathcal{C}(0,\overline{p}_{S_{\tau}})) - \text{Pr}(\mathbf{Y}\in \mathcal{C}(0,\overline{p}_{S_{\tau-1}})) \\
\label{part1_step1_r_y_q2_prime_2}
=&\text{Pr}(\mathbf{Y}\in \mathcal{B}_{S_{\tau}})-\text{Pr}(\mathbf{Y}\in \mathcal{B}_{S_{\tau-1}}) \\
\label{part1_step1_r_y_q2_prime_3}
\leq&\text{Pr}(\mathbf{Y}\in \mathcal{B}_{S_{\tau}})- \frac{(\tau-1)\cdot \text{Pr}(\mathbf{Y}\in  \mathcal{B}_{S_{\tau}})}{\tau} \\
=& \frac{\text{Pr}(\mathbf{Y}\in \mathcal{B}_{S_{\tau}})}{\tau}
\end{align}
We obtain Equation~\ref{part1_step1_r_y_q2_prime_2} from Equation~\ref{part1_step1_r_y_q2_prime_1} based on  Lemma~\ref{lemma_lambda_1_q_2}, and Equation~\ref{part1_step1_r_y_q2_prime_3} from Equation~\ref{part1_step1_r_y_q2_prime_2} based on  Equation~\ref{compare_others_with_j}. Therefore, we have the following: 
\begin{align}
(r_y(q_1^{\prime},q_2)-\frac{\text{Pr}(\mathbf{Y}\in \mathcal{B}_{S_{\tau}})}{\tau})\cdot(r_y(q_1,q_2^{\prime})-\frac{\text{Pr}(\mathbf{Y}\in \mathcal{B}_{S_{\tau}})}{\tau})\leq 0  
\end{align}
Thus, there exists $(q_1^{\tau},q_2^{\tau})$ such that $r_x(q_1^{\tau},q_2^{\tau})=\overline{p}_{b_{\tau}},r_y(q_1^{\tau},q_2^{\tau})=\frac{\text{Pr}(\mathbf{Y}\in \mathcal{B}_{S_{\tau}})}{\tau}$ based on Lemma~\ref{lemma_of_exists}. Then, we have the following based on the definition of $r_x,r_y$: 
\begin{align}
    &\text{Pr}(\mathbf{X}\in \mathcal{C}^{\prime}(q_1,q_2)\cap \mathcal{C}(q_1^{\tau},q_2^{\tau})) = \overline{p}_{b_{\tau}} \\
   & \text{Pr}(\mathbf{Y}\in \mathcal{C}^{\prime}(q_1,q_2)\cap \mathcal{C}(q_1^{\tau},q_2^{\tau})) = \frac{\text{Pr}(\mathbf{Y}\in \mathcal{B}_{S_{\tau}})}{\tau} 
\end{align}
Finally, we let $\mathcal{C}_{b_{\tau}}= \mathcal{C}^{\prime}(q_1,q_2)\cap\mathcal{C}(q_1^{\tau},q_2^{\tau})$, which meets our goal.

\myparatight{Step II}
Assuming we can find $\{(q_1^{\tau},q_2^{\tau}),(q_1^{\tau-1}, q_2^{\tau-1}), \cdots,(q_1^{\tau-m+1},q_2^{\tau-m+1})\}$  ($\forall j \in[\tau-m+1,\tau],q_1\leq q_1^j\leq q_2^j \leq q_2$) where $1\leq m \leq \tau-1$ such that $\forall j \in [\tau-m+1 , \tau]$,  we have:
\begin{align}
\label{part1_stepII_condition_x_part}
&\text{Pr}(\mathbf{X}\in \mathcal{C}_{b_j})=\overline{p}_{b_j} \\
\label{part1_stepII_condition_y_part}
&\text{Pr}(\mathbf{Y}\in \mathcal{C}_{b_j})=\frac{\text{Pr}(\mathbf{Y}\in \mathcal{B}_{S_{\tau}})}{\tau}, 
\end{align}
as well as the following:
\begin{align}
\label{part1_stepII_condition_of_disjoint}
\forall j,t \in [\tau-m+1,\tau]\text{ and }j\neq t,\mathcal{C}_{b_j}\cap\mathcal{C}_{b_t}=\emptyset
\end{align}
We denote $e=\tau-m$. We  show  we can find $\mathcal{C}_{b_e}$ such that we have: 
\begin{align}
\label{part1_step2_equation_e_1}
      &\text{Pr}(\mathbf{X}\in \mathcal{C}_{b_e})=\overline{p}_{b_e} \\
\label{part1_step2_equation_e_2}
    &\text{Pr}(\mathbf{Y}\in \mathcal{C}_{b_e})=\frac{\text{Pr}(\mathbf{Y}\in \mathcal{B}_{S_{\tau}})}{\tau} \\
\label{part1_step2_equation_e_3}
&\forall j\in [e+1,\tau], \mathcal{C}_{b_e}\cap\mathcal{C}_{b_j}=\emptyset
\end{align}
We let $q_1=0,q_2=\overline{p}_{S_{\tau}}$ and denote
\begin{align}
     \mathcal{C}^{\prime}(q_1,q_2)=\mathcal{C}(q_1,q_2)\setminus\cup_{j=e+1}^{\tau}\mathcal{C}_{b_j}
\end{align}
We have the following: 
\begin{align}
&r_x(q_1,q_2) \\
=&\text{Pr}(\mathbf{X}\in\mathcal{C}^{\prime}(q_1,q_2)\cap\mathcal{C}(q_1,q_2)) \\
\label{part1_stepII_rx_q1_q2_value_1}
=&\text{Pr}(\mathbf{X}\in \mathcal{C}(q_1,q_2))-\text{Pr}(\mathbf{X}\in \cup_{j=e+1}^{\tau}\mathcal{C}_{b_j})  \\
\label{part1_stepII_rx_q1_q2_value_2}
=&\text{Pr}(\mathbf{X}\in \mathcal{C}(q_1,q_2))-\sum_{j=e+1}^{\tau}\text{Pr}(\mathbf{X}\in \mathcal{C}_{b_j})  \\
\label{part1_stepII_rx_q1_q2_value_3}
=&\sum_{j=1}^{\tau}\overline{p}_{b_j}-\sum_{j=e+1}^{\tau}\overline{p}_{b_j} \\
=&\sum_{j=1}^{e}\overline{p}_{b_j}
\end{align}
The Equation~\ref{part1_stepII_rx_q1_q2_value_2} from~\ref{part1_stepII_rx_q1_q2_value_1} is based on the Equation~\ref{part1_stepII_condition_of_disjoint}, and the Equation~\ref{part1_stepII_rx_q1_q2_value_3} from~\ref{part1_stepII_rx_q1_q2_value_2} is based the Equation~\ref{equation_0_definition_of_the_region} and~\ref{part1_stepII_condition_x_part}. 
Furthermore, we have the following: 
\begin{align}
&r_y(q_1,q_2)  \\
=&\text{Pr}(\mathbf{Y}\in\mathcal{C}^{\prime}(q_1,q_2)\cap\mathcal{C}(q_1,q_2)) \\
=&\text{Pr}(\mathbf{Y}\in \mathcal{C}(q_1,q_2))-\text{Pr}(\mathbf{Y}\in \cup_{j=e+1}^{\tau}\mathcal{C}_{b_j}) \\
\label{part1_stepII_ry_q1_q2_equation1}
=&\text{Pr}(\mathbf{Y}\in \mathcal{C}(q_1,q_2))-\sum_{j=e+1}^{\tau}\text{Pr}(\mathbf{Y}\in \mathcal{C}_{b_j}) \\
\label{part1_stepII_ry_q1_q2_equation2}
=&\text{Pr}(\mathbf{Y}\in\mathcal{B}_{S_{\tau}})- \frac{m\cdot\text{Pr}( \mathbf{Y}\in \mathcal{B}_{S_{\tau}})}{\tau} \\
=& \frac{(\tau-m)\cdot\text{Pr}( \mathbf{Y}\in \mathcal{B}_{S_{\tau}})}{\tau} 
\end{align}
The Equation~\ref{part1_stepII_ry_q1_q2_equation2} from~\ref{part1_stepII_ry_q1_q2_equation1} is because $\mathcal{C}(q_1,q_2)=\mathcal{B}_{S_{\tau}}$ and the Equation~\ref{part1_stepII_condition_y_part}. 
We have  $\overline{p}_{b_e}\leq  r_x(q_1,q_2)$. Therefore, based on Lemma~\ref{lemma_of_exists}, there exist $q_1^{\prime},q_2^{\prime}$ such that: 
\begin{align}
\label{part1_step2_rx_q1_q2prime_value_compute}
 & r_x(q_1,q_2^{\prime})=\overline{p}_{b_e} \\
&r_x(q_1^{\prime},q_2)=\overline{p}_{b_e}
\end{align}
We have:
\begin{align}
\label{part1_step2_r_y_lamba1_q_2prime_prior}
& r_y(q_1,q_2^{\prime}) \\
\label{part1_step2_r_y_lamba1_q_2prime}
\geq&r_{y}(q_1,q_2)\cdot \frac{1}{\lceil r_x(q_1,q_2)/r_x(q_1,q_2^{\prime}) \rceil} \\
\label{part1_step2_r_y_lamba1_q_2prime_later}
\geq&r_{y}(q_1,q_2)\cdot \frac{1}{\lceil (\sum_{j=1}^{e}\overline{p}_{b_j})/\overline{p}_{b_e} \rceil} \\
\label{part1_step2_r_y_lamba1_q_2prime_last}
\geq&  \frac{\text{Pr}( \mathbf{Y}\in \mathcal{B}_{S_{\tau}})}{\tau}
\end{align}
Equation~\ref{part1_step2_r_y_lamba1_q_2prime} from~\ref{part1_step2_r_y_lamba1_q_2prime_prior} is based on Lemma~\ref{lemma_of_compare_middle_value}, Equation~\ref{part1_step2_r_y_lamba1_q_2prime_later} from~\ref{part1_step2_r_y_lamba1_q_2prime} is based on Equation~\ref{part1_stepI_rx_q1_q2_value_compute} and ~\ref{part1_step2_rx_q1_q2prime_value_compute}, and Equation~\ref{part1_step2_r_y_lamba1_q_2prime_last} from~\ref{part1_step2_r_y_lamba1_q_2prime_later} is obtained from Equation~\ref{part1_stepI_ry_q1_q2_value_compute} and the fact that $\lceil (\sum_{j=1}^{e}\overline{p}_{b_j})/\overline{p}_{b_e} \rceil \geq \tau$. 
Next, we show:
\begin{align}
    r_y(q_1^{\prime},q_2)\leq \frac{\text{Pr}( \mathbf{Y}\in \mathcal{B}_{S_{\tau}})}{\tau}
\end{align}
In particular, we consider two scenarios.

\emph{Scenario 1)}. $q_1^{\prime}\geq \min_{j=\tau-m+1}^{\tau}q_{1}^{j}$. We denote $t=\argmin_{j=\tau-m+1}^{\tau}q_1^j$. We let $\mathcal{C}_w=\mathcal{C}_{b_t}\subseteq \mathcal{C}(q_1^t,q_2^t)$. As  $q_1^{\prime}\geq q_1^t, q_2 \geq q_2^t$ and $r_x(q_1^{\prime},q_2)=\overline{p}_{b_e}\leq \text{Pr}(\mathbf{X}\in\mathcal{C}_w)=\overline{p}_{b_t}$, we have the following based on Lemma~\ref{lemma_of_compare_disjoint}: 
\begin{align}
  r_y(q_1^{\prime},q_2)  \leq \text{Pr}(\mathbf{Y}\in\mathcal{C}_w) = \frac{\text{Pr}( \mathbf{Y}\in \mathcal{B}_{S_{\tau}})}{\tau}
\end{align}

\emph{Scenario 2)}. $q_1^{\prime}< \min_{j=\tau-m+1}^{\tau}q_{1}^{j}$.  We have the following:
\begin{align}
\mathcal{C}^{\prime}(q_1,q_2)\cap\mathcal{C}(q_1,q_1^{\prime})=\mathcal{C}(q_1,q_1^{\prime})
\end{align}
Furthermore, we have:  
\begin{align}
&r_x(q_1,q_1^{\prime}) \\
=&r_x(q_1,q_2) - r_x(q_1^{\prime},q_2) \\
=&\sum_{j=1}^{e}\overline{p}_{b_j}-\overline{p}_{b_e} \\
=&\sum_{j=1}^{e-1}\overline{p}_{b_j} 
\end{align}
Moreover, we have  $r_x(q_1,q_1^{\prime})=q_1^{\prime}-q_1$ from Lemma~\ref{lemma_lambda_1_q_2}. The above two should be equal. Thus, we have $q_1^{\prime}=\sum_{j=1}^{e-1}\overline{p}_{b_j}=\overline{p}_{S_{\tau-m-1}}$ since $e=\tau-m$. we have: 
\begin{align}
&r_y(q_1^{\prime},q_2) \\
\label{part2_step2_scenario2_r_y_1}
 =&r_y(q_1,q_2)-r_y(q_1,q_1^{\prime}) \\
 \label{part2_step2_scenario2_r_y_2}
 =&r_y(q_1,q_2) - \text{Pr}(\mathbf{Y}\in \mathcal{B}_{S_{\tau-m-1}}) \\
\leq&\frac{(\tau-m)\cdot\text{Pr}( \mathbf{Y}\in \mathcal{B}_{S_{\tau}})}{\tau} -  \frac{(\tau-m-1)\cdot\text{Pr}( \mathbf{Y}\in \mathcal{B}_{S_{\tau}})}{\tau} \\
=&\frac{\text{Pr}( \mathbf{Y}\in \mathcal{B}_{S_{\tau}})}{\tau}
\end{align}
We obtain Equation~\ref{part2_step2_scenario2_r_y_2} from Equation~\ref{part2_step2_scenario2_r_y_1} based on Lemma~\ref{lemma_lambda_1_q_2}. 

Therefore, we have the following in both scenarios: 
\begin{align}
   &r_y(q_1,q_2^{\prime}) \geq  \frac{\text{Pr}( \mathbf{Y}\in \mathcal{B}_{S_{\tau}})}{\tau} \\
   & r_y(q_1^{\prime},q_2) \leq  \frac{\text{Pr}( \mathbf{Y}\in \mathcal{B}_{S_{\tau}})}{\tau}
\end{align}
Based on Lemma~\ref{lemma_of_exists}, there exist $q_1^e,q_2^e$ such that $r_x(q_1^{e},q_2^{e})=\overline{p}_{b_{\tau}},r_y(q_1^{e},q_2^{e})=\frac{\text{Pr}(\mathbf{Y}\in \mathcal{B}_{S_{\tau}})}{\tau}$. Then, we have the following based on the definition of $r_x,r_y$: 
\begin{align}
      &\text{Pr}(\mathbf{X}\in \mathcal{C}^{\prime}(q_1,q_2)\cap\mathcal{C}(q_1^{e},q_2^{e}))=\overline{p}_{b_e} \\
    &\text{Pr}(\mathbf{Y}\in \mathcal{C}^{\prime}(q_1,q_2)\cap\mathcal{C}(q_1^{e},q_2^{e}))=\frac{\text{Pr}(\mathbf{Y}\in \mathcal{B}_{S_{\tau}})}{\tau}
\end{align}
We let $\mathcal{C}_{e}=\mathcal{C}^{\prime}(q_1,q_2)\cap\mathcal{C}(q_1^{e},q_2^{e})$. 
From the definition of $\mathcal{C}^{\prime}(q_1,q_2)$, we have  $\forall j\in [e+1,\tau], \mathcal{C}^{\prime}(q_1,q_2)\cap\mathcal{C}_{b_j}=\emptyset$. Thus, we have  $\forall j\in [e+1,\tau], \mathcal{C}_{b_e}\cap\mathcal{C}_{b_j}=\emptyset$ since $\mathcal{C}_{b_e}\subseteq \mathcal{C}^{\prime}(q_1,q_2)$. 

Therefore, we reach our goal by Mathematical Induction, i.e., for $\forall j \in [1,\tau]$, we have: 
\begin{align}
&\text{Pr}(\mathbf{X}\in \mathcal{C}_{b_j})=\overline{p}_{b_j} \\
 &   \text{Pr}(\mathbf{Y}\in \mathcal{C}_{b_j}) = \frac{\text{Pr}(\mathbf{Y}\in\mathcal{B}_{S_{\tau}})}{\tau} 
\end{align}
We can also verify that $\cup_{j=1}^{\tau}\mathcal{C}_{b_j}=\mathcal{B}_{S_{\tau}}$. 

Next, we show our proof based on Mathematical Induction for the other part, i.e., $\mathcal{B}_{S_k}\setminus\mathcal{B}_{S_{\tau}}$. Our construction process is similar to the above first part but has subtle differences.  

\myparatight{Step I}
Let $q_1= \sum_{j=1}^{\tau}\overline{p}_{b_j}$ and $q_2=\sum_{j=1}^{k}\overline{p}_{b_j}$. 
We define: 
\begin{align}
  \mathcal{C}^{\prime}(q_1,q_2)= \mathcal{C}(q_1,q_2) =\mathcal{B}_{S_k}\setminus\mathcal{B}_{S_{\tau}}
\end{align}
 Then, we have: 
\begin{align}
\label{part2_step1_compute_rx_q1_q2_value}
 &r_x(q_1,q_2) =q_2 - q_1 = \sum_{j=\tau+1}^{k}\overline{p}_{b_j}  \\
 &r_y(q_1,q_2) \\
 =&\text{Pr}(\mathbf{Y}\in \mathcal{C}^{\prime}(q_1,q_2)\cap\mathcal{C}(q_1,q_2)) \\
 =&\text{Pr}(\mathbf{Y}\in \mathcal{C}(0,q_2))-\text{Pr}(\mathbf{Y}\in \mathcal{C}(q_1,q_2)) \\
 \label{part2_step1_r_y_lambda_1_lambda_2_1}
 =&\text{Pr}(\mathbf{Y}\in \mathcal{B}_{S_k})-\text{Pr}(\mathbf{Y}\in \mathcal{B}_{S_{\tau}}) \\
 \label{part2_step1_r_y_lambda_1_lambda_2_2}
 \geq&\frac{k\cdot \text{Pr}(\mathbf{Y}\in \mathcal{B}_{S_{\tau}})}{\tau}-\text{Pr}(\mathbf{Y}\in \mathcal{B}_{S_{\tau}}) \\
 \label{part2_step1_r_y_lambda_1_lambda_2_3}
 =&\frac{(k-\tau)\cdot \text{Pr}(\mathbf{Y}\in \mathcal{B}_{S_{\tau}})}{\tau}
\end{align}
The Equation~\ref{part2_step1_compute_rx_q1_q2_value} is based on the fact that $ \mathcal{C}^{\prime}(q_1,q_2)= \mathcal{C}(q_1,q_2)$ and Definition~\ref{definition_for_the_region}, and we obtain Equation~\ref{part2_step1_r_y_lambda_1_lambda_2_1} from~\ref{part2_step1_r_y_lambda_1_lambda_2_2} based on Equation~\ref{compare_others_with_j}. 
We have $\overline{p}_{b_k}\leq  r_x(q_1,q_2)$. Therefore, based on Lemma~\ref{lemma_of_exists}, we know that there exists $q_1^{\prime}=q_2-\overline{p}_{b_k}, q_2^{\prime}=q_1+\overline{p}_{b_k}$ such that: 
\begin{align}
   & r_x(q_1,q_2^{\prime})=\overline{p}_{b_k} \\
    &r_x(q_1^{\prime},q_2)=\overline{p}_{b_k}
\end{align}

We consider two scenarios. 

\emph{Scenario 1)}. In this scenario, we consider $r_y(q_1^{\prime},q_2) > \frac{\text{Pr}(\mathbf{Y}\in \mathcal{B}_{S_{\tau}})}{\tau}$. 
We  let $q_1^k=q_1^{\prime},q_2^k=q_2$, i.e., we have $\mathcal{C}_{b_k}=\mathcal{C}(q_1^{\prime},q_2)\cap \mathcal{C}^{\prime}(q_1,q_2)$. Then, we have: 
\begin{align}
    &\text{Pr}(\mathbf{X}\in\mathcal{C}_{b_k})=r_x(q_1^{\prime},q_2)= \overline{p}_{b_k} \\
    &\text{Pr}(\mathbf{Y}\in\mathcal{C}_{b_k})=r_y(q_1^{\prime},q_2)\geq\frac{\text{Pr}(\mathbf{Y}\in \mathcal{B}_{S_{\tau}})}{\tau}    
\end{align}
Therefore, we have the following: 
\begin{align}
&\text{Pr}(\mathbf{Y}\in \mathcal{C}(q_1,q_2)\setminus  \mathcal{C}_{b_k}) \\
=&\text{Pr}(\mathbf{Y}\in\mathcal{C}(q_1,q_1^{\prime})) \\
=&\text{Pr}(\mathbf{Y}\in\mathcal{C}(0,q_1^{\prime})) -\text{Pr}(\mathbf{Y}\in\mathcal{C}(0,q_1)) \\
=&\text{Pr}(\mathbf{Y}\in\mathcal{B}_{S_{k-1}})-\text{Pr}(\mathbf{Y}\in \mathcal{B}_{S_{\tau}})   \\
\geq&\frac{(k-\tau-1)\cdot\text{Pr}(\mathbf{Y}\in \mathcal{B}_{S_{\tau}})}{\tau}  
\end{align}

\emph{Scenario 2)}. In this scenario, we consider $r_y(q_1^{\prime},q_2) \leq \frac{\text{Pr}(\mathbf{Y}\in \mathcal{B}_{S_{\tau}})}{\tau}$. 
We have the following: 
\begin{align}
\label{part2_step1_ry_q1_q2prime_value_compare_equation1}
&r_y(q_1,q_2^{\prime}) \\
\label{part2_step1_ry_q1_q2prime_value_compare_equation2}
\geq& r_y(q_1,q_2)\cdot \frac{1}{\lceil r_x(q_1,q_2)/r_x(q_1,q_2^{\prime}) \rceil} \\
\label{part2_step1_ry_q1_q2prime_value_compare_equation3}
\geq&\frac{(k-\tau)\cdot \text{Pr}(\mathbf{Y}\in \mathcal{B}_{S_{\tau}})}{\tau}\cdot \frac{1}{k-\tau} \\
=&\frac{\text{Pr}(\mathbf{Y}\in \mathcal{B}_{S_{\tau}})}{\tau}    
\end{align}
We obtain Equation~\ref{part2_step1_ry_q1_q2prime_value_compare_equation2} from~\ref{part2_step1_ry_q1_q2prime_value_compare_equation1} via Lemma~\ref{lemma_of_compare_middle_value}, and we obtain Equation~\ref{part2_step1_ry_q1_q2prime_value_compare_equation3} from~\ref{part2_step1_ry_q1_q2prime_value_compare_equation2} based on Equation~\ref{part2_step1_compute_rx_q1_q2_value} to~\ref{part2_step1_r_y_lambda_1_lambda_2_3} and the fact $r_x(q_1,q_2)/r_x(q_1,q_2^{\prime}) = \frac{\sum_{j=\tau+1}^{k}\overline{p}_{b_j}}{\overline{p}_{b_{\tau+1}}}\geq k-\tau$. We have $(r_y(q_1^{\prime},q_2) - \frac{\text{Pr}(\mathbf{Y}\in \mathcal{B}_{S_{\tau}})}{\tau})\cdot(r_y(q_1,q_2^{\prime}) - \frac{\text{Pr}(\mathbf{Y}\in \mathcal{B}_{S_{\tau}})}{\tau})\leq 0$. Therefore, 
from Lemma~\ref{lemma_of_exists}, we know that there exist $(q_1^k,q_2^k)$ such that: 
\begin{align}
    &\text{Pr}(\mathbf{X}\in\mathcal{C}_{b_k})=r_x(q_1^{k},q_2^{k})=\overline{p}_{b_k} \\
    &\text{Pr}(\mathbf{Y}\in\mathcal{C}_{b_k})=r_y(q_1^{k},q_2^{k})=\frac{\text{Pr}(\mathbf{Y}\in \mathcal{B}_{S_{\tau}})}{\tau}    
\end{align}
We also have the following: 
\begin{align}
&\text{Pr}(\mathbf{Y}\in \mathcal{C}(q_1,q_2)\setminus  \mathcal{C}_{b_k}) \\
=&r_y(q_1,q_2)-r_y(q_1^k,q_2^k) \\
\geq&\frac{(k-\tau-1)\cdot\text{Pr}(\mathbf{Y}\in \mathcal{B}_{S_{\tau}})}{\tau}  
\end{align}
Similarly, we let $\mathcal{C}_{b_k}=\mathcal{C}^{\prime}(q_1,q_2)\cap\mathcal{C}(q_1^k,q_2^k)$. 

Based on the conditions of our constructions in the two scenarios, we know that if $\text{Pr}(\mathbf{Y}\in\mathcal{B}_{b_k})>\frac{\text{Pr}(\mathbf{Y}\in \mathcal{B}_{S_{\tau}})}{\tau}$, then we  have   $q_2^k=q_2$.

\myparatight{Step II}
We  show that if we can find $(q_1^k,q_2^k),\cdots,(q_1^{k-m+1},q_2^{k-m+1})$ where $m\in [1,k-\tau -1]$ and $\mathcal{C}_{b_j},\forall j \in [k,k-m+1]$
such that:
\begin{align}
    &\text{Pr}(\mathbf{X}\in \mathcal{C}_{b_j})=\overline{p}_{b_j} \\
    &\text{Pr}(\mathbf{Y}\in \mathcal{C}_{b_j})\geq\frac{\text{Pr}(\mathbf{Y}\in \mathcal{B}_{S_{\tau}})}{\tau} \\
    &\text{Pr}(\mathbf{Y}\in \mathcal{C}(q_1,q_2)\setminus \cup_{t=k-m+1}^{k}\mathcal{C}_{b_t}) 
\geq\frac{(k-\tau-m)\cdot\text{Pr}(\mathbf{Y}\in \mathcal{B}_{S_{\tau}})}{\tau} 
\end{align}
Then,  we can find $(q_1^{k-m},q_2^{k-m})$ such that: 
\begin{align}
    &\text{Pr}(\mathbf{X}\in \mathcal{C}_{b_{k-m}})=\overline{p}_{b_{k-m}} \\
    &\text{Pr}(\mathbf{Y}\in \mathcal{C}_{b_{k-m}})\geq\frac{\text{Pr}(\mathbf{Y}\in \mathcal{B}_{S_{\tau}})}{\tau} \\
    &\text{Pr}(\mathbf{Y}\in \mathcal{C}(q_1,q_2)\setminus \cup_{t=k-m}^{k}\mathcal{C}_{b_t}) 
\geq\frac{(k-\tau-m -1)\cdot\text{Pr}(\mathbf{Y}\in \mathcal{B}_{S_{\tau}})}{\tau} 
\end{align}
For simplicity, we denote $e=k-m$, we let $q_1= \sum_{j=1}^{\tau}\overline{p}_{b_j}$ and $q_2=\sum_{j=1}^{k}\overline{p}_{b_j}$, and we define:  
\begin{align}
     \mathcal{C}^{\prime}(q_1,q_2)=\mathcal{C}(q_1,q_2)\setminus\cup_{j=e+1}^{k}\mathcal{C}_{b_j}
\end{align}
Then, we have: 
\begin{align}
\label{part2_step2_rx_q1_q2_value_equation0}
&r_x(q_1,q_2)=\text{Pr}(\mathbf{X}\in \mathcal{C}(q_1,q_2)\setminus \cup_{j=e+1}^{k}\mathcal{C}_{b_j})=\sum_{j=\tau +1}^{e}\overline{p}_{b_j} \\
\label{part2_step2_rx_q1_q2_value_equation1}
&r_y(q_1,q_2)=\text{Pr}(\mathbf{Y}\in \mathcal{C}(q_1,q_2)\setminus \cup_{j=e+1}^{k}\mathcal{C}_{b_j})\geq\frac{(k-\tau-m)\cdot\text{Pr}(\mathbf{Y}\in \mathcal{B}_{S_{\tau}})}{\tau}
\end{align}
We have $\overline{p}_{b_e}\leq  r_x(q_1,q_2)$. Therefore, based on Lemma~\ref{lemma_of_exists}, we know that there exists $q_1^{\prime}, q_2^{\prime}$ such that: 
\begin{align}
  &r_x(q_1^{\prime},q_2)=\overline{p}_{b_e} \\
  &  r_x(q_1,q_2^{\prime})=\overline{p}_{b_e}
\end{align}
Similarly, we consider two scenarios: 

\emph{Scenario 1)}. In this scenario, we consider that the following holds: 
\begin{align}
    \label{constrainsts_for_q1prime}
    r_y(q_1^{\prime},q_2) > \frac{\text{Pr}(\mathbf{Y}\in \mathcal{B}_{S_{\tau}})}{\tau}    
\end{align} 
We let $q_1^e=q_1^{\prime},q_2^e=q_2$, i.e., $\mathcal{C}_{b_e}=\mathcal{C}(q_1^{\prime},q_2)\cap \mathcal{C}^{\prime}(q_1,q_2)$. 
Then, we have: 
\begin{align}
    &\text{Pr}(\mathbf{X}\in \mathcal{C}_{b_e})=r_x(q_1^{\prime},q_2)=\overline{p}_{b_e} \\
    &\text{Pr}(\mathbf{Y}\in \mathcal{C}_{b_e})=r_y(q_1^{\prime},q_2)\geq\frac{\text{Pr}(\mathbf{Y}\in \mathcal{B}_{S_{\tau}})}{\tau}
\end{align}
We note that we  have $q_1^{\prime}\leq \min_{j=e+1}^{k}q_1^i$ in this scenario. Otherwise, Equation~\ref{constrainsts_for_q1prime} will not hold based on Lemma~\ref{lemma_of_compare_disjoint}. We give a short proof. 

Assuming $q_1^{\prime}> \min_{j=k-m+1}^{k}q_{1}^{j}$. We denote $w=\argmin_{j=k-m+1}^{k}q_1^j$. We let $\mathcal{C}_w=\mathcal{C}_{b_w}\subseteq \mathcal{C}(q_1^w,q_2^w)$. Note that in this case, we  have $q_2^w<q_2$ because $q_2^w=q_2$ and $q_1^{\prime}> q_1^w$ cannot hold at the same time as long as $r_y(q_1^{\prime},q_2)>0$. Thus, we have  $\text{Pr}(\mathbf{Y}\in\mathcal{C}_w)
=\frac{\text{Pr}( \mathbf{Y}\in \mathcal{B}_{S_{\tau}})}{\tau}$ because  if $\text{Pr}(\mathbf{Y}\in\mathcal{C}_w)
>\frac{\text{Pr}( \mathbf{Y}\in \mathcal{B}_{S_{\tau}})}{\tau}$, we  have $q_2^w=q_2$.
As we have  $q_1^{\prime}> q_1^w, q_2 > q_2^w$ and $r_x(q_1^{\prime},q_2)=\overline{p}_{b_e}\leq \text{Pr}(\mathbf{X}\in\mathcal{C}_w)=\overline{p}_{b_w}$. We have the following based on Lemma~\ref{lemma_of_compare_disjoint}: 
\begin{align}
    \label{contradiction_q1prime}
  r_y(q_1^{\prime},q_2)  \leq \text{Pr}(\mathbf{Y}\in\mathcal{C}_w) = \frac{\text{Pr}( \mathbf{Y}\in \mathcal{B}_{S_{\tau}})}{\tau}
\end{align}
Since Equation \ref{constrainsts_for_q1prime} and Equation \ref{contradiction_q1prime} cannot hold at the same time, the assumption $q_1^{\prime}>\min_{j=k-m+1}^{k}q_{1}^{j}$ must be wrong. Therefore, we  have $q_1^{\prime}\leq \min_{j=k-m+1}^{k}q_{1}^{j}$.

Based on $q_1^{\prime}\leq \min_{j=k-m+1}^{k}q_{1}^{j}$, we have the following: 
\begin{align}
    \mathcal{C}^{\prime}(q_1,q_2)\cap\mathcal{C}(q_1,q_1^{\prime})=\mathcal{C}(q_1,q_1^{\prime})
\end{align}
Therefore, we have  $r_x(q_1,q_1^{\prime})=q_1^{\prime}-q_1$ from Definition~\ref{definition_for_the_region}. Moreover, we have the following: 
\begin{align}
   & r_x(q_1,q_1^{\prime}) \\
    =&r_x(q_1,q_2)-r_x(q_1^{\prime},q_2) \\
    =&\sum_{j=\tau +1}^{e}\overline{p}_{b_j}-\overline{p_{b_e}} \\
    =&\sum_{j=\tau +1}^{e-1}\overline{p}_{b_j}
\end{align}
The above two should be equal. Therefore, we have  $q_1^{\prime}=\sum_{j=\tau +1}^{e-1}p_{b_j}+q_1=\sum_{j=1}^{e-1}p_{b_j}$. Recall that we let $\mathcal{C}_{b_e}=\mathcal{C}^{\prime}(q_1,q_2)\cap\mathcal{C}(q_1^{\prime},q_2)$. Thus, we have: 
\begin{align}
 &\text{Pr}(\mathbf{Y}\in \mathcal{C}(q_1,q_2)\setminus \cup_{j=e}^{k}\mathcal{C}_{b_j}) \\
 =&\text{Pr}(\mathbf{Y}\in \mathcal{C}^{\prime}(q_1,q_2)\setminus \mathcal{C}_{b_{e}}) \\
 =&\text{Pr}(\mathbf{Y}\in \mathcal{C}(q_1,q_1^{\prime}))\\
 =&\text{Pr}(\mathbf{Y}\in \mathcal{C}(0,q_1^{\prime}))-\text{Pr}(\mathbf{Y}\in \mathcal{C}(0,q_1)) \\
 =&\text{Pr}(\mathbf{Y}\in \mathcal{B}_{S_{e-1}})- \text{Pr}(\mathbf{Y}\in \mathcal{B}_{S_{\tau}}) \\
\geq&\frac{(k-\tau -m-1)\cdot\text{Pr}(\mathbf{Y}\in \mathcal{B}_{S_{\tau}})}{\tau}   
\end{align}

\emph{Scenario 2)}. In this scenario, we consider that the following holds: 
\begin{align}
   r_y(q_1^{\prime},q_2) \leq \frac{\text{Pr}(\mathbf{Y}\in \mathcal{B}_{S_{\tau}})}{\tau} 
\end{align}
Note that we have: 
\begin{align}
\label{part2_step2_ry_q1_q2prime_value_compare_equation1}
&r_y(q_1,q_2^{\prime}) \\
\label{part2_step2_ry_q1_q2prime_value_compare_equation2}
\geq&r_y(q_1,q_2)\cdot \frac{1}{\lceil r_x(q_1,q_2)/r_x(q_1,q_2^{\prime}) \rceil} \\
\label{part2_step2_ry_q1_q2prime_value_compare_equation3}
\geq&r_y(q_1,q_2)\cdot \frac{1}{k-\tau -m} \\
\geq&\frac{\text{Pr}(\mathbf{Y}\in \mathcal{B}_{S_{\tau}})}{\tau}       
\end{align}
We obtain Equation~\ref{part2_step2_ry_q1_q2prime_value_compare_equation2} from~\ref{part2_step2_ry_q1_q2prime_value_compare_equation1} via Lemma~\ref{lemma_of_compare_middle_value}, and we obtain Equation~\ref{part2_step2_ry_q1_q2prime_value_compare_equation3} from~\ref{part2_step2_ry_q1_q2prime_value_compare_equation2} based on Equation~\ref{part2_step2_rx_q1_q2_value_equation1} and the fact $r_x(q_1,q_2)/r_x(q_1,q_2^{\prime}) = \frac{\sum_{j=\tau+1}^{e}\overline{p}_{b_j}}{\overline{p}_{b_{\tau+1}}}\geq k-\tau-m$. We have $(r_y(q_1^{\prime},q_2) - \frac{\text{Pr}(\mathbf{Y}\in \mathcal{B}_{S_{\tau}})}{\tau})\cdot(r_y(q_1,q_2^{\prime}) - \frac{\text{Pr}(\mathbf{Y}\in \mathcal{B}_{S_{\tau}})}{\tau})\leq 0$.
Based on Lemma~\ref{lemma_of_exists}, we can find $(q_1^e,q_2^e)$ such that we have: 
\begin{align}
    &\text{Pr}(\mathbf{X}\in \mathcal{C}^{\prime}(q_1,q_2)\cap\mathcal{C}(q_1^e,q_2^e))=r_x(q_1^{e},q_2^{e})=\overline{p}_{b_e} \\
    &\text{Pr}(\mathbf{Y}\in \mathcal{C}^{\prime}(q_1,q_2)\cap\mathcal{C}(q_1^e,q_2^e))=r_y(q_1^{e},q_2^{e})=\frac{\text{Pr}(\mathbf{Y}\in \mathcal{B}_{S_{\tau}})}{\tau}    
\end{align}
We let $\mathcal{C}_{b_e}=\mathcal{C}^{\prime}(q_1,q_2)\cap\mathcal{C}(q_1^e,q_2^e)$. We also have the following: 
\begin{align}
 &\text{Pr}(\mathbf{Y}\in \mathcal{C}(q_1,q_2)\setminus \cup_{j=e}^{k}\mathcal{C}_{b_j}) \\
= &\text{Pr}(\mathbf{Y}\in \mathcal{C}^{\prime}(q_1,q_2)\setminus \mathcal{C}_{b_e}) \\
=&r_y(q_1,q_2)-r_y(q_1^e,q_2^e) \\
 \geq & \frac{(k-\tau-m)\cdot\text{Pr}(\mathbf{Y}\in \mathcal{B}_{S_{\tau}})}{\tau} - \frac{\text{Pr}(\mathbf{Y}\in \mathcal{B}_{S_{\tau}})}{\tau} \\
\geq &\frac{(k-\tau-m-1)\cdot\text{Pr}(\mathbf{Y}\in \mathcal{B}_{S_{\tau}})}{\tau}   
\end{align}
Similar to \textbf{Step I}, we still hold the conclusion that if $\text{Pr}(\mathbf{Y}\in\mathcal{C}_{b_e})>\frac{\text{Pr}(\mathbf{Y}\in \mathcal{B}_{S_{\tau}})}{\tau}$, we  have   $q_2^e=q_2$. 
Then, we can apply \emph{Mathematical Induction} to reach the conclusion. 
Also, we can verify $\cup_{j=\tau +1}^{k}\mathcal{C}_{b_j}=\mathcal{B}_{S_k}\setminus\mathcal{B}_{S_{\tau}}$. 
\end{proof}

\section{Proof of Proposition~\ref{proposition_1}}
\label{proof_of_proposition_1}

The function \textsc{SampleUndernoise}$(f,k,\sigma,\mathbf{x},n,\alpha)$ works as follows: we first draw $n$ random noise from $\mathcal{N}(0,\sigma^2 I)$, i.e., $\epsilon_1,\epsilon_2,\cdots,\epsilon_n$. Then, we compute the values: $\forall i \in [1,c], c_i = \sum_{j=1}^{n}\mathbb{I}(f(\mathbf{x}+\epsilon_j)=i)$. 
The function \textsc{Binompvalue}$(n_{c_t},n_{c_t}+n_{c_{t+1}},p)$ returns the result of p-value of the two-sided hypothesis test for $n_{c_t} \sim Bin(n_{c_t}+n_{c_{t+1}},p)$.

\myparatight{Proposition 1} With probability at least $1-\alpha$ over the randomness in \textsc{Predict}, if  \textsc{Predict} returns a set $T$ (i.e., does not ABSTAIN), then we have $g_{k}(\mathbf{x})=T$.
\begin{proof}
We aim to compute the probability that \textsc{Predict} returns a set which not equals to $g_{k}(\mathbf{x})$, which  happens if and only if $g_{k}(\mathbf{x})\neq T$ and \textsc{Predict} doesn't abstain. Specifically, we have: 
\begin{align}
    & \text{Pr}(\textsc{Predict} \text{ returns a set } \neq g_{k}(\mathbf{x})) \\
    =& \text{Pr}(g_{k}(\mathbf{x})\neq T,\textsc{Predict} \text{ doesn't abstain}) \\
    =& \text{Pr}(g_{k}(\mathbf{x})\neq T)\cdot \text{Pr}(\textsc{Predict} \text{ doesn't abstain}|g_{k}(\mathbf{x})\neq T) \\
    \leq & \text{Pr}(\textsc{Predict} \text{ doesn't abstain}|g_{k}(\mathbf{x})\neq T)
\end{align}
Theorem 1 in~\citet{hung2019rank} shows the above conditional probability is as follows: 
\begin{align}
    \text{Pr}(\textsc{Predict} \text{ doesn't abstain}|g_{k}(\mathbf{x})\neq T) \leq \alpha
\end{align}
Therefore, we reach the conclusion. 
\end{proof}

\section{Proof of Proposition~\ref{proposition_certify}}
\label{proof_of_proposition_2}
\myparatight{Proposition 2} With probability at least $1 -  \alpha$ over the randomness in $\textsc{Certify}$, if $\textsc{Certify}$ returns a radius $\underline{R_l}$ (i.e., does not ABSTAIN), then we have $l \in g_{k}(\mathbf{x}+\delta), \forall \lnorm{\delta}_2 < \underline{R_l}$. 

\begin{proof}
From the definition of \textsc{BinoCP} and \textsc{SimuEM}, we know the probability that the following inequalities simultaneously hold is at least $1-\alpha$ over the sampling of counts: 
\begin{align}
    \underline{p_i} \le \text{Pr}(f(x+\epsilon)=i) \text{ if } i=l\\
    \overline{p}_{i} \ge \text{Pr}(f(x+\epsilon)=i) \text{ if } i\ne l
\end{align}
Then, with the returned bounds, we can invoke Theorem \ref{theorem_of_certified_radius} to obtain the robustness guarantee if the calculated radius is larger than 0. Note that otherwise \textsc{Certify}  abstains.
\end{proof}

\section{Certified top-$k$ accuracy}
\label{certified_test_set_accuracy}
We show how to derive a  lower bound of the certified top-$k$ accuracy based on the {approximate certified top-$k$ accuracy}. The process is similar to that~\citet{cohen2019certified} used to derive a lower bound of the certified top-$1$ accuracy based on the {approximate certified top-$1$ accuracy}. 
Specifically, we have the following lemma from~\citet{cohen2019certified}.
\begin{restatable}{lem}{approximateexactcerradiustheorem}
\label{approximate_exact_certified_radius_theorem}
Let $z_i$ be a binary variable and $Y_i$ be a Bernoulli random variable.  Suppose if $z_i=1$, then $\text{Pr}(Y_i=1)\leq \alpha$. Then, for any $\rho > 0$, with probability at least $ 1 - \rho$, we have the following: 
\begin{align}
    \frac{1}{m}\sum_{i=1}^{m}z_i \geq \frac{1}{1 - \alpha}(\frac{\sum_{i=1}^{m}Y_i}{m}-\alpha - \sqrt{\frac{2\alpha (1-\alpha)\log\frac{1}{\rho}}{m}}-\frac{\log(\frac{1}{\rho})}{3m})
\end{align}
\end{restatable}
\begin{proof}
Please refer to~\citet{cohen2019certified}. 
\end{proof}

Assuming we have a test dataset $D_{test}=\{(\mathbf{x}_1,y_1),(\mathbf{x}_2,y_2),\cdots,(\mathbf{x_m},y_m)\}$ as well as a radius $r$. We define the following indicate value: 
\begin{align}
    a_i = \mathbb{I}(y_i \in g_{k}(\mathbf{x}_i + \delta)), \forall ||\delta||_2 <r
\end{align}
Then, the certified top-$k$ accuracy of the smoothed classifier $g$ at radius $r$ can be computed as $\frac{1}{m}\sum_{i=1}^{m}a_i$. 
For each sample $\mathbf{x}_i$,  we run the \textsc{Certify} function with $1 - \alpha$ confidence level and we use a random variable $Y_i$ to denote that the function \textsc{Certify} returns a radius bigger than $r$. From \textbf{Proposition 2}, we know: 
\begin{align}
\label{fact_to_exploit}
    \text{Pr}(Y_i = 1 )\leq \alpha, \text{ if } a_i = 1
\end{align}
The approximate certified top-$k$ accuracy of the smoothed classifier at radius $r$ is $\frac{1}{m}\sum_{i=1}^{m}Y_i$. 
Then, we can use Lemma~\ref{approximate_exact_certified_radius_theorem} to obtain a lower bound of $\frac{1}{m}\sum_{i=1}^{m}a_i$. Specifically, for any $\rho > 0$, with probability at least $1-\rho$ over the randomness of \textsc{Certify}, we have: 
\begin{align}
    \frac{1}{m}\sum_{i=1}^{m}a_i \geq \frac{1}{1 - \alpha}(\frac{\sum_{i=1}^{m}Y_i}{m}-\alpha - \sqrt{\frac{2\alpha (1-\alpha)\log\frac{1}{\rho}}{m}}-\frac{\log(\frac{1}{\rho})}{3m})
\end{align}
We can see that the difference between the certified top-$k$ accuracy and the {approximate certified top-$k$ accuracy} is negligible when $\alpha$ is small. 

%% file: paper.bbl
\begin{thebibliography}{72}
\providecommand{\natexlab}[1]{#1}
\providecommand{\url}[1]{\texttt{#1}}
\expandafter\ifx\csname urlstyle\endcsname\relax
  \providecommand{\doi}[1]{doi: #1}\else
  \providecommand{\doi}{doi: \begingroup \urlstyle{rm}\Url}\fi

\bibitem[Anil et~al.(2019)Anil, Lucas, and Grosse]{anil2018sorting}
Cem Anil, James Lucas, and Roger Grosse.
\newblock Sorting out lipschitz function approximation.
\newblock In \emph{International Conference on Machine Learning}, 2019.

\bibitem[Athalye \& Carlini(2018)Athalye and Carlini]{athalye2018robustness}
Anish Athalye and Nicholas Carlini.
\newblock On the robustness of the cvpr 2018 white-box adversarial example
  defenses.
\newblock \emph{The Bright and Dark Sides of Computer Vision: Challenges and
  Opportunities for Privacy and Security (CV-COPS)}, 2018.

\bibitem[Athalye et~al.(2018)Athalye, Carlini, and
  Wagner]{athalye2018obfuscated}
Anish Athalye, Nicholas Carlini, and David Wagner.
\newblock Obfuscated gradients give a false sense of security: Circumventing
  defenses to adversarial examples.
\newblock In \emph{International Conference on Machine Learning}, pp.\
  274--283, 2018.

\bibitem[Buckman et~al.(2018)Buckman, Roy, Raffel, and
  Goodfellow]{buckman2018thermometer}
Jacob Buckman, Aurko Roy, Colin Raffel, and Ian Goodfellow.
\newblock Thermometer encoding: One hot way to resist adversarial examples.
\newblock In \emph{ICLR}, 2018.

\bibitem[Bunel et~al.(2018)Bunel, Turkaslan, Torr, Kohli, and
  Mudigonda]{bunel2018unified}
Rudy~R Bunel, Ilker Turkaslan, Philip Torr, Pushmeet Kohli, and Pawan~K
  Mudigonda.
\newblock A unified view of piecewise linear neural network verification.
\newblock In \emph{Advances in Neural Information Processing Systems}, pp.\
  4790--4799, 2018.

\bibitem[Cao \& Gong(2017)Cao and Gong]{cao2017mitigating}
Xiaoyu Cao and Neil~Zhenqiang Gong.
\newblock Mitigating evasion attacks to deep neural networks via region-based
  classification.
\newblock In \emph{Proceedings of the 33rd Annual Computer Security
  Applications Conference}, pp.\  278--287. ACM, 2017.

\bibitem[Carlini \& Wagner(2017{\natexlab{a}})Carlini and
  Wagner]{carlini2017adversarial}
Nicholas Carlini and David Wagner.
\newblock Adversarial examples are not easily detected: Bypassing ten detection
  methods.
\newblock In \emph{Proceedings of the 10th ACM Workshop on Artificial
  Intelligence and Security}, pp.\  3--14. ACM, 2017{\natexlab{a}}.

\bibitem[Carlini \& Wagner(2017{\natexlab{b}})Carlini and
  Wagner]{carlini2017towards}
Nicholas Carlini and David Wagner.
\newblock Towards evaluating the robustness of neural networks.
\newblock In \emph{2017 IEEE Symposium on Security and Privacy (SP)}, pp.\
  39--57. IEEE, 2017{\natexlab{b}}.

\bibitem[Carlini et~al.(2017)Carlini, Katz, Barrett, and
  Dill]{carlini2017provably}
Nicholas Carlini, Guy Katz, Clark Barrett, and David~L Dill.
\newblock Provably minimally-distorted adversarial examples.
\newblock \emph{arXiv preprint arXiv:1709.10207}, 2017.

\bibitem[Cheng et~al.(2017)Cheng, N{\"u}hrenberg, and Ruess]{cheng2017maximum}
Chih-Hong Cheng, Georg N{\"u}hrenberg, and Harald Ruess.
\newblock Maximum resilience of artificial neural networks.
\newblock In \emph{International Symposium on Automated Technology for
  Verification and Analysis}, pp.\  251--268. Springer, 2017.

\bibitem[Cisse et~al.(2017)Cisse, Bojanowski, Grave, Dauphin, and
  Usunier]{cisse2017parseval}
Moustapha Cisse, Piotr Bojanowski, Edouard Grave, Yann Dauphin, and Nicolas
  Usunier.
\newblock Parseval networks: Improving robustness to adversarial examples.
\newblock In \emph{Proceedings of the 34th International Conference on Machine
  Learning-Volume 70}, pp.\  854--863. JMLR. org, 2017.

\bibitem[{Clarifai}()]{clarifai_demo}
{Clarifai}.
\newblock \url{https://www.clarifai.com/demo}.
\newblock July 2019.

\bibitem[Cohen et~al.(2019)Cohen, Rosenfeld, and Kolter]{cohen2019certified}
Jeremy~M Cohen, Elan Rosenfeld, and J~Zico Kolter.
\newblock Certified adversarial robustness via randomized smoothing.
\newblock \emph{arXiv preprint arXiv:1902.02918}, 2019.

\bibitem[Croce et~al.(2018)Croce, Andriushchenko, and Hein]{croce2018provable}
Francesco Croce, Maksym Andriushchenko, and Matthias Hein.
\newblock Provable robustness of relu networks via maximization of linear
  regions.
\newblock In \emph{Proceedings of the 22nd International Conference on
  Artificial Intelligence and Statistics}, 2018.

\bibitem[Deng et~al.(2009)Deng, Dong, Socher, Li, Li, and
  Fei-Fei]{deng2009imagenet}
Jia Deng, Wei Dong, Richard Socher, Li-Jia Li, Kai Li, and Li~Fei-Fei.
\newblock Imagenet: A large-scale hierarchical image database.
\newblock In \emph{2009 IEEE conference on computer vision and pattern
  recognition}, pp.\  248--255. Ieee, 2009.

\bibitem[Dhillon et~al.(2018)Dhillon, Azizzadenesheli, Lipton, Bernstein,
  Kossaifi, Khanna, and Anandkumar]{dhillon2018stochastic}
Guneet~S Dhillon, Kamyar Azizzadenesheli, Zachary~C Lipton, Jeremy Bernstein,
  Jean Kossaifi, Aran Khanna, and Anima Anandkumar.
\newblock Stochastic activation pruning for robust adversarial defense.
\newblock In \emph{ICLR}, 2018.

\bibitem[Dutta et~al.(2017)Dutta, Jha, Sanakaranarayanan, and
  Tiwari]{dutta2017output}
Souradeep Dutta, Susmit Jha, Sriram Sanakaranarayanan, and Ashish Tiwari.
\newblock Output range analysis for deep neural networks.
\newblock \emph{arXiv preprint arXiv:1709.09130}, 2017.

\bibitem[Dvijotham et~al.(2018{\natexlab{a}})Dvijotham, Gowal, Stanforth,
  Arandjelovic, O'Donoghue, Uesato, and Kohli]{dvijotham2018training}
Krishnamurthy Dvijotham, Sven Gowal, Robert Stanforth, Relja Arandjelovic,
  Brendan O'Donoghue, Jonathan Uesato, and Pushmeet Kohli.
\newblock Training verified learners with learned verifiers.
\newblock \emph{arXiv preprint arXiv:1805.10265}, 2018{\natexlab{a}}.

\bibitem[Dvijotham et~al.(2018{\natexlab{b}})Dvijotham, Stanforth, Gowal, Mann,
  and Kohli]{dvijotham2018dual}
Krishnamurthy Dvijotham, Robert Stanforth, Sven Gowal, Timothy~A Mann, and
  Pushmeet Kohli.
\newblock A dual approach to scalable verification of deep networks.
\newblock In \emph{UAI}, pp.\  550--559, 2018{\natexlab{b}}.

\bibitem[Ehlers(2017)]{ehlers2017formal}
Ruediger Ehlers.
\newblock Formal verification of piece-wise linear feed-forward neural
  networks.
\newblock In \emph{International Symposium on Automated Technology for
  Verification and Analysis}, pp.\  269--286. Springer, 2017.

\bibitem[Fischetti \& Jo(2018)Fischetti and Jo]{fischetti2018deep}
Matteo Fischetti and Jason Jo.
\newblock Deep neural networks and mixed integer linear optimization.
\newblock \emph{Constraints}, 23:\penalty0 296--309, 2018.

\bibitem[Gehr et~al.(2018)Gehr, Mirman, Drachsler-Cohen, Tsankov, Chaudhuri,
  and Vechev]{gehr2018ai2}
Timon Gehr, Matthew Mirman, Dana Drachsler-Cohen, Petar Tsankov, Swarat
  Chaudhuri, and Martin Vechev.
\newblock Ai2: Safety and robustness certification of neural networks with
  abstract interpretation.
\newblock In \emph{2018 IEEE Symposium on Security and Privacy (SP)}, pp.\
  3--18. IEEE, 2018.

\bibitem[Goodfellow et~al.(2015)Goodfellow, Shlens, and
  Szegedy]{goodfellow2014explaining}
Ian~J Goodfellow, Jonathon Shlens, and Christian Szegedy.
\newblock Explaining and harnessing adversarial examples.
\newblock In \emph{International Conference on Learning Representations}, 2015.

\bibitem[Goodman(1965)]{goodman1965simultaneous}
Leo~A Goodman.
\newblock On simultaneous confidence intervals for multinomial proportions.
\newblock \emph{Technometrics}, 7\penalty0 (2):\penalty0 247--254, 1965.

\bibitem[{Google Cloud Vision}()]{Google_Cloud_Vision}
{Google Cloud Vision}.
\newblock \url{https://cloud.google.com/vision/}.
\newblock July 2019.

\bibitem[Gouk et~al.(2018)Gouk, Frank, Pfahringer, and
  Cree]{gouk2018regularisation}
Henry Gouk, Eibe Frank, Bernhard Pfahringer, and Michael Cree.
\newblock Regularisation of neural networks by enforcing lipschitz continuity.
\newblock \emph{arXiv preprint arXiv:1804.04368}, 2018.

\bibitem[Gowal et~al.(2018)Gowal, Dvijotham, Stanforth, Bunel, Qin, Uesato,
  Mann, and Kohli]{gowal2018effectiveness}
Sven Gowal, Krishnamurthy Dvijotham, Robert Stanforth, Rudy Bunel, Chongli Qin,
  Jonathan Uesato, Timothy Mann, and Pushmeet Kohli.
\newblock On the effectiveness of interval bound propagation for training
  verifiably robust models.
\newblock \emph{arXiv preprint arXiv:1810.12715}, 2018.

\bibitem[Guo et~al.(2018)Guo, Rana, Cisse, and Van
  Der~Maaten]{guo2018countering}
Chuan Guo, Mayank Rana, Moustapha Cisse, and Laurens Van Der~Maaten.
\newblock Countering adversarial images using input transformations.
\newblock In \emph{ICLR}, 2018.

\bibitem[Huang et~al.(2017)Huang, Kwiatkowska, Wang, and Wu]{huang2017safety}
Xiaowei Huang, Marta Kwiatkowska, Sen Wang, and Min Wu.
\newblock Safety verification of deep neural networks.
\newblock In \emph{International Conference on Computer Aided Verification},
  pp.\  3--29. Springer, 2017.

\bibitem[Hung et~al.(2019)Hung, Fithian, et~al.]{hung2019rank}
Kenneth Hung, William Fithian, et~al.
\newblock Rank verification for exponential families.
\newblock \emph{The Annals of Statistics}, 47\penalty0 (2):\penalty0 758--782,
  2019.

\bibitem[Jia \& Gong(2018)Jia and Gong]{jia2018attriguard}
Jinyuan Jia and Neil~Zhenqiang Gong.
\newblock {AttriGuard}: A practical defense against attribute inference attacks
  via adversarial machine learning.
\newblock In \emph{USENIX Security Symposium}, 2018.

\bibitem[Katz et~al.(2017)Katz, Barrett, Dill, Julian, and
  Kochenderfer]{katz2017reluplex}
Guy Katz, Clark Barrett, David~L Dill, Kyle Julian, and Mykel~J Kochenderfer.
\newblock Reluplex: An efficient smt solver for verifying deep neural networks.
\newblock In \emph{International Conference on Computer Aided Verification},
  pp.\  97--117. Springer, 2017.

\bibitem[Krizhevsky \& Hinton(2009)Krizhevsky and
  Hinton]{krizhevsky2009learning}
Alex Krizhevsky and Geoffrey Hinton.
\newblock Learning multiple layers of features from tiny images.
\newblock Technical report, Citeseer, 2009.

\bibitem[Kurakin et~al.(2017)Kurakin, Goodfellow, and
  Bengio]{Kurakin2017AdversarialML}
Alexey Kurakin, Ian~J. Goodfellow, and Samy Bengio.
\newblock Adversarial machine learning at scale.
\newblock In \emph{International Conference on Learning Representations}, 2017.

\bibitem[Lecuyer et~al.(2019)Lecuyer, Atlidakis, Geambasu, Hsu, and
  Jana]{lecuyer2018certified}
Mathias Lecuyer, Vaggelis Atlidakis, Roxana Geambasu, Daniel Hsu, and Suman
  Jana.
\newblock Certified robustness to adversarial examples with differential
  privacy.
\newblock In \emph{IEEE Symposium on Security and Privacy (SP)}, 2019.

\bibitem[Lee et~al.(2019)Lee, Yuan, Chang, and Jaakkola]{lee2019tight}
Guang-He Lee, Yang Yuan, Shiyu Chang, and Tommi Jaakkola.
\newblock Tight certificates of adversarial robustness for randomly smoothed
  classifiers.
\newblock In \emph{Advances in Neural Information Processing Systems}, pp.\
  4911--4922, 2019.

\bibitem[Li et~al.(2018)Li, Chen, Wang, and Carin]{li2018second}
Bai Li, Changyou Chen, Wenlin Wang, and Lawrence Carin.
\newblock Second-order adversarial attack and certifiable robustness.
\newblock \emph{arXiv preprint arXiv:1809.03113}, 2018.

\bibitem[Liu et~al.(2018)Liu, Cheng, Zhang, and Hsieh]{liu2018towards}
Xuanqing Liu, Minhao Cheng, Huan Zhang, and Cho-Jui Hsieh.
\newblock Towards robust neural networks via random self-ensemble.
\newblock In \emph{Proceedings of the European Conference on Computer Vision
  (ECCV)}, pp.\  369--385, 2018.

\bibitem[Liu et~al.(2019)Liu, Li, Wu, and Hsieh]{liu2018adv}
Xuanqing Liu, Yao Li, Chongruo Wu, and Cho-Jui Hsieh.
\newblock Adv-bnn: Improved adversarial defense through robust bayesian neural
  network.
\newblock In \emph{ICLR}, 2019.

\bibitem[Lomuscio \& Maganti(2017)Lomuscio and Maganti]{lomuscio2017approach}
Alessio Lomuscio and Lalit Maganti.
\newblock An approach to reachability analysis for feed-forward relu neural
  networks.
\newblock \emph{arXiv preprint arXiv:1706.07351}, 2017.

\bibitem[Ma et~al.(2018)Ma, Li, Wang, Erfani, Wijewickrema, Schoenebeck, Song,
  Houle, and Bailey]{ma2018characterizing}
Xingjun Ma, Bo~Li, Yisen Wang, Sarah~M Erfani, Sudanthi Wijewickrema, Grant
  Schoenebeck, Dawn Song, Michael~E Houle, and James Bailey.
\newblock Characterizing adversarial subspaces using local intrinsic
  dimensionality.
\newblock In \emph{ICLR}, 2018.

\bibitem[Madry et~al.(2018)Madry, Makelov, Schmidt, Tsipras, and
  Vladu]{madry2017towards}
Aleksander Madry, Aleksandar Makelov, Ludwig Schmidt, Dimitris Tsipras, and
  Adrian Vladu.
\newblock Towards deep learning models resistant to adversarial attacks.
\newblock In \emph{International Conference on Learning Representations}, 2018.

\bibitem[Meng \& Chen(2017)Meng and Chen]{meng2017magnet}
Dongyu Meng and Hao Chen.
\newblock Magnet: a two-pronged defense against adversarial examples.
\newblock In \emph{Proceedings of the 2017 ACM SIGSAC Conference on Computer
  and Communications Security}, pp.\  135--147. ACM, 2017.

\bibitem[Metzen et~al.(2017)Metzen, Genewein, Fischer, and
  Bischoff]{metzen2017detecting}
Jan~Hendrik Metzen, Tim Genewein, Volker Fischer, and Bastian Bischoff.
\newblock On detecting adversarial perturbations.
\newblock In \emph{ICLR}, 2017.

\bibitem[Mirman et~al.(2018)Mirman, Gehr, and Vechev]{mirman2018differentiable}
Matthew Mirman, Timon Gehr, and Martin Vechev.
\newblock Differentiable abstract interpretation for provably robust neural
  networks.
\newblock In \emph{International Conference on Machine Learning}, pp.\
  3575--3583, 2018.

\bibitem[Na et~al.(2018)Na, Ko, and Mukhopadhyay]{na2017cascade}
Taesik Na, Jong~Hwan Ko, and Saibal Mukhopadhyay.
\newblock Cascade adversarial machine learning regularized with a unified
  embedding.
\newblock In \emph{ICLR}, 2018.

\bibitem[Neyman \& Pearson(1933)Neyman and Pearson]{neyman1933ix}
Jerzy Neyman and Egon~Sharpe Pearson.
\newblock On the problem of the most efficient tests of statistical hypotheses.
\newblock \emph{Philosophical Transactions of the Royal Society of London.
  Series A, Containing Papers of a Mathematical or Physical Character},
  231\penalty0 (694-706):\penalty0 289--337, 1933.

\bibitem[Papernot et~al.(2016)Papernot, McDaniel, Wu, Jha, and
  Swami]{papernot2016distillation}
Nicolas Papernot, Patrick McDaniel, Xi~Wu, Somesh Jha, and Ananthram Swami.
\newblock Distillation as a defense to adversarial perturbations against deep
  neural networks.
\newblock In \emph{2016 IEEE Symposium on Security and Privacy (SP)}, pp.\
  582--597. IEEE, 2016.

\bibitem[Pinot et~al.(2019)Pinot, Meunier, Araujo, Kashima, Yger, Gouy-Pailler,
  and Atif]{pinot2019theoretical}
Rafael Pinot, Laurent Meunier, Alexandre Araujo, Hisashi Kashima, Florian Yger,
  Cedric Gouy-Pailler, and Jamal Atif.
\newblock Theoretical evidence for adversarial robustness through
  randomization.
\newblock In \emph{Advances in Neural Information Processing Systems}, pp.\
  11838--11848, 2019.

\bibitem[Raghunathan et~al.(2018{\natexlab{a}})Raghunathan, Steinhardt, and
  Liang]{raghunathan2018certified}
Aditi Raghunathan, Jacob Steinhardt, and Percy Liang.
\newblock Certified defenses against adversarial examples.
\newblock In \emph{International Conference on Learning Representations},
  2018{\natexlab{a}}.

\bibitem[Raghunathan et~al.(2018{\natexlab{b}})Raghunathan, Steinhardt, and
  Liang]{raghunathan2018semidefinite}
Aditi Raghunathan, Jacob Steinhardt, and Percy~S Liang.
\newblock Semidefinite relaxations for certifying robustness to adversarial
  examples.
\newblock In \emph{Advances in Neural Information Processing Systems}, pp.\
  10877--10887, 2018{\natexlab{b}}.

\bibitem[Salman et~al.(2019)Salman, Li, Razenshteyn, Zhang, Zhang, Bubeck, and
  Yang]{salman2019provably}
Hadi Salman, Jerry Li, Ilya Razenshteyn, Pengchuan Zhang, Huan Zhang, Sebastien
  Bubeck, and Greg Yang.
\newblock Provably robust deep learning via adversarially trained smoothed
  classifiers.
\newblock In \emph{Advances in Neural Information Processing Systems}, pp.\
  11289--11300, 2019.

\bibitem[Samangouei et~al.(2018)Samangouei, Kabkab, and
  Chellappa]{samangouei2018defense}
Pouya Samangouei, Maya Kabkab, and Rama Chellappa.
\newblock Defense-gan: Protecting classifiers against adversarial attacks using
  generative models.
\newblock In \emph{ICLR}, 2018.

\bibitem[Schott et~al.(2019)Schott, Rauber, Bethge, and
  Brendel]{schott2019towards}
Lukas Schott, Jonas Rauber, Matthias Bethge, and Wieland Brendel.
\newblock Towards the first adversarially robust neural network model on mnist.
\newblock In \emph{ICLR}, 2019.

\bibitem[Singh et~al.(2018)Singh, Gehr, Mirman, P{\"u}schel, and
  Vechev]{singh2018fast}
Gagandeep Singh, Timon Gehr, Matthew Mirman, Markus P{\"u}schel, and Martin
  Vechev.
\newblock Fast and effective robustness certification.
\newblock In \emph{Advances in Neural Information Processing Systems}, pp.\
  10802--10813, 2018.

\bibitem[Sison \& Glaz(1995)Sison and Glaz]{sison1995simultaneous}
Cristina~P Sison and Joseph Glaz.
\newblock Simultaneous confidence intervals and sample size determination for
  multinomial proportions.
\newblock \emph{Journal of the American Statistical Association}, 90\penalty0
  (429):\penalty0 366--369, 1995.

\bibitem[Song et~al.(2017)Song, Kim, Nowozin, Ermon, and
  Kushman]{song2017pixeldefend}
Yang Song, Taesup Kim, Sebastian Nowozin, Stefano Ermon, and Nate Kushman.
\newblock Pixeldefend: Leveraging generative models to understand and defend
  against adversarial examples.
\newblock In \emph{ICLR}, 2017.

\bibitem[Song et~al.(2018)Song, Kim, Nowozin, Ermon, and
  Kushman]{song2018pixeldefend}
Yang Song, Taesup Kim, Sebastian Nowozin, Stefano Ermon, and Nate Kushman.
\newblock Pixeldefend: Leveraging generative models to understand and defend
  against adversarial examples.
\newblock In \emph{ICLR}, 2018.

\bibitem[Svoboda et~al.(2019)Svoboda, Masci, Monti, Bronstein, and
  Guibas]{svoboda2018peernets}
Jan Svoboda, Jonathan Masci, Federico Monti, Michael~M Bronstein, and Leonidas
  Guibas.
\newblock Peernets: Exploiting peer wisdom against adversarial attacks.
\newblock In \emph{ICLR}, 2019.

\bibitem[Szegedy et~al.(2014)Szegedy, Zaremba, Sutskever, Bruna, Erhan,
  Goodfellow, and Fergus]{Szegedy14}
Christian Szegedy, Wojciech Zaremba, Ilya Sutskever, Joan Bruna, Dumitru Erhan,
  Ian Goodfellow, and Rob Fergus.
\newblock Intriguing properties of neural networks.
\newblock In \emph{ICLR}, 2014.

\bibitem[Tjeng et~al.(2018)Tjeng, Xiao, and Tedrake]{tjeng2017evaluating}
Vincent Tjeng, Kai Xiao, and Russ Tedrake.
\newblock Evaluating robustness of neural networks with mixed integer
  programming.
\newblock In \emph{ICML}, 2018.

\bibitem[Tram{\`e}r et~al.(2018)Tram{\`e}r, Kurakin, Papernot, Goodfellow,
  Boneh, and McDaniel]{tramer2018ensemble}
Florian Tram{\`e}r, Alexey Kurakin, Nicolas Papernot, Ian Goodfellow, Dan
  Boneh, and Patrick McDaniel.
\newblock Ensemble adversarial training: Attacks and defenses.
\newblock In \emph{ICLR}, 2018.

\bibitem[Tsuzuku et~al.(2018)Tsuzuku, Sato, and Sugiyama]{tsuzuku2018lipschitz}
Yusuke Tsuzuku, Issei Sato, and Masashi Sugiyama.
\newblock Lipschitz-margin training: Scalable certification of perturbation
  invariance for deep neural networks.
\newblock In \emph{Advances in Neural Information Processing Systems}, pp.\
  6541--6550, 2018.

\bibitem[Uesato et~al.(2018)Uesato, O'Donoghue, Kohli, and
  Oord]{uesato2018adversarial}
Jonathan Uesato, Brendan O'Donoghue, Pushmeet Kohli, and Aaron Oord.
\newblock Adversarial risk and the dangers of evaluating against weak attacks.
\newblock In \emph{International Conference on Machine Learning}, pp.\
  5032--5041, 2018.

\bibitem[Wang et~al.(2018{\natexlab{a}})Wang, Chen, Abdou, and
  Jana]{wang2018mixtrain}
Shiqi Wang, Yizheng Chen, Ahmed Abdou, and Suman Jana.
\newblock Mixtrain: Scalable training of formally robust neural networks.
\newblock \emph{arXiv preprint arXiv:1811.02625}, 2018{\natexlab{a}}.

\bibitem[Wang et~al.(2018{\natexlab{b}})Wang, Pei, Whitehouse, Yang, and
  Jana]{wang2018efficient}
Shiqi Wang, Kexin Pei, Justin Whitehouse, Junfeng Yang, and Suman Jana.
\newblock Efficient formal safety analysis of neural networks.
\newblock In \emph{Advances in Neural Information Processing Systems}, pp.\
  6367--6377, 2018{\natexlab{b}}.

\bibitem[Weng et~al.(2018)Weng, Zhang, Chen, Song, Hsieh, Boning, Dhillon, and
  Daniel]{weng2018towards}
Tsui-Wei Weng, Huan Zhang, Hongge Chen, Zhao Song, Cho-Jui Hsieh, Duane Boning,
  Inderjit~S Dhillon, and Luca Daniel.
\newblock Towards fast computation of certified robustness for relu networks.
\newblock \emph{arXiv preprint arXiv:1804.09699}, 2018.

\bibitem[Wong \& Kolter(2018)Wong and Kolter]{wong2018provable}
Eric Wong and Zico Kolter.
\newblock Provable defenses against adversarial examples via the convex outer
  adversarial polytope.
\newblock In \emph{International Conference on Machine Learning}, pp.\
  5283--5292, 2018.

\bibitem[Wong et~al.(2018)Wong, Schmidt, Metzen, and Kolter]{wong2018scaling}
Eric Wong, Frank Schmidt, Jan~Hendrik Metzen, and J~Zico Kolter.
\newblock Scaling provable adversarial defenses.
\newblock In \emph{Advances in Neural Information Processing Systems}, pp.\
  8400--8409, 2018.

\bibitem[Xie et~al.(2018)Xie, Wang, Zhang, Ren, and Yuille]{xie2018mitigating}
Cihang Xie, Jianyu Wang, Zhishuai Zhang, Zhou Ren, and Alan Yuille.
\newblock Mitigating adversarial effects through randomization.
\newblock In \emph{ICLR}, 2018.

\bibitem[Xu et~al.(2018)Xu, Evans, and Qi]{xu2017feature}
Weilin Xu, David Evans, and Yanjun Qi.
\newblock Feature squeezing: Detecting adversarial examples in deep neural
  networks.
\newblock In \emph{NDSS}, 2018.

\bibitem[Zhang et~al.(2018)Zhang, Weng, Chen, Hsieh, and
  Daniel]{zhang2018efficient}
Huan Zhang, Tsui-Wei Weng, Pin-Yu Chen, Cho-Jui Hsieh, and Luca Daniel.
\newblock Efficient neural network robustness certification with general
  activation functions.
\newblock In \emph{Advances in Neural Information Processing Systems}, pp.\
  4939--4948, 2018.

\end{thebibliography}
